\documentclass[twoside,11pt]{article}

\usepackage[preprint]{main}

\usepackage{url}
\usepackage{bbm}
\usepackage{array}
\usepackage{xcolor}
\usepackage{amsmath}
\usepackage{amssymb}
\usepackage{booktabs}
\usepackage{placeins}
\usepackage{multirow}
\usepackage{makecell}
\usepackage{algorithm}
\usepackage[noend]{algpseudocode}
\usepackage[figuresright]{rotating}

\usepackage{tikz}
\usetikzlibrary{calc}
\usetikzlibrary{arrows.meta,positioning}

\usepackage{tabularx}

\newcolumntype{L}[1]{>{\raggedright\arraybackslash}m{#1}}
\newcolumntype{C}[1]{>{\centering\arraybackslash}m{#1}}
\newcolumntype{Y}{>{\centering\arraybackslash}X}

\newcommand{\X}{\mathcal{X}}
\newcommand{\Y}{\mathcal{Y}}
\newcommand{\Z}{\mathcal{Z}}
\newcommand{\1}{\mathbbm{1}}
\newcommand{\PN}{P_N}

\arxivheading{}{2026}{1-\pageref{LastPage}}{}{}{}{Bach C. Le, Tung V. Dao, Binh T. Nguyen, and Hong T. M. Chu}
\ShortHeadings{Wasserstein Robust Classification Error}{Bach C. Le, Tung V. Dao, Binh T. Nguyen, and Hong T. M. Chu}
\firstpageno{1}

\begin{document}

\title{Wasserstein Robust Classification Error:\\
Surrogate Loss Misalignment and Distributional Attacks}

\author{%
\name Bach C. Le \email bach.lc@vinuni.edu.vn\\
\name Tung V. Dao \email tung.dv@vinuni.edu.vn\\
\name Binh T. Nguyen \email binh.nt2@vinuni.edu.vn\\
\name Hong T. M. Chu \email hong.ctm@vinuni.edu.vn\\
\addr College of Engineering and Computer Science\\
\addr VinUniversity\\
\addr Hanoi, Vietnam
}

\maketitle

\begin{abstract}
Wasserstein distributionally robust optimization (WDRO) provides a framework for evaluating model robustness through the worst-case expectation in the Wasserstein ball around a reference distribution.
We prove that ordering hypotheses by increasing WDRO risk is equivalent to ordering them by decreasing robust accuracy if and only if the loss function is a positive multiple of the zero-one loss plus a label-dependent offset.
For a class of loss functions including cross-entropy, we show that a hypothesis with perfect robust accuracy can conversely have a higher WDRO risk than another hypothesis with near zero robust accuracy.
To evaluate robust accuracy directly, we develop Wasserstein distributional attacks (WDA), which add a weight generation stage after candidate support search to optimize the allocation of empirical mass under the Wasserstein constraint, rather than assigning each data point's full mass to one returned attack point as existing methods do.
We further introduce a support-search algorithm for WDA that finds nearby misclassified inputs.
On CIFAR-10, CIFAR-100, and ImageNet, WDA using our search algorithm achieves the lowest adversarial accuracy in most tested settings.
Our code is available at \url{https://github.com/OLab-Repo/WDA}.
\end{abstract}

\vspace{0.6\baselineskip}


\section{Introduction}
\label{sec:introduction}

Classification robustness concerns how prediction error changes under perturbations of the inputs or their distribution.
The susceptibility of neural networks to small adversarial perturbations has motivated extensive work on adversarial attacks and robustness evaluation \citep{szegedy2014intriguing,goodfellow2015explaining,kurakin2018adversarial,carlini2017towards,croce2020reliable}.
Related studies examine robustness to common corruptions \citep{hendrycks2019benchmarking}, predictive uncertainty under dataset shift \citep{ovadia2019can}, and classification performance under natural distribution shifts \citep{taori2020measuring,koh2021wilds}.
Wasserstein distributionally robust optimization (WDRO) provides a framework to quantify robustness under distributional perturbations through the worst-case expected loss over a Wasserstein ball around a reference distribution \citep{kuhn2019wasserstein,kuhn2025distributionally}.
For a hypothesis \(\theta\) represented by a scoring function \(f_\theta\), the WDRO risk associated with a loss function \(\ell\) is
\begin{equation}
\mathcal R_p(\theta;\varepsilon,\ell)
=
\sup_{Q:\,W_p(\PN,Q)\le\varepsilon}
\mathbb E_Q\!\left[\ell(Y,f_\theta(X))\right].
\label{eq:wdro-risk}
\end{equation}
Here \(\PN\) is the reference distribution formed from the data points, \(Q\) ranges over joint distributions of inputs \(X\) and labels \(Y\), and \(W_p\) is the Wasserstein distance of order \(p\in[1,\infty]\).
Transportation is constrained to preserve labels, and \(\varepsilon\ge0\) specifies the perturbation budget.
Taking \(\ell=\1\) in \eqref{eq:wdro-risk} gives our primary object, the Wasserstein robust classification error,
\begin{equation}
\mathcal R_p(\theta;\varepsilon,\1)
=
\sup_{Q:\,W_p(\PN,Q)\le\varepsilon}
\mathbb E_Q\!\left[\1(Y,f_\theta(X))\right],
\label{eq:robust-classification-error-intro}
\end{equation}
where \(\1\) is the classification error loss, equal to one for an incorrect prediction and zero otherwise.
This value is the supremum of the misclassification probability over the admissible distributions.
The Wasserstein robust accuracy is \(1-\mathcal R_p(\theta;\varepsilon,\1)\).
For the detailed setup, see Section~\ref{sec:notation}.

\paragraph{Wasserstein distributionally robust optimization.}
Strong duality results provide representations of \(\mathcal R_p(\theta;\varepsilon,\ell)\) in \eqref{eq:wdro-risk} for \(1\le p<\infty\), under suitable assumptions on the loss and transportation cost \citep{blanchet2019transport,gao2023distributionally,zhang2025duality}.
These representations involve a scalar Lagrange multiplier and pointwise maximization of the loss penalized by transportation cost.
For structured classes of losses and support sets, they yield finite-dimensional convex reformulations \citep{mohajerinesfahani2018datadriven}.
In supervised learning, connections with regularization yield tractable convex formulations for linear regression and linear classification models under suitable assumptions on the loss and transportation cost \citep{shafieezadehabadeh2015distributionally,shafieezadehabadeh2019regularization}.
\citet{chu2026wasserstein} establish sufficient conditions for exact regularization reformulations, with applications to linear binary classification using losses such as the logistic and hinge losses.
For neural classifiers, WDRO has also motivated adversarial training methods that optimize model parameters against distributional perturbations \citep{sinha2018certifying,bui2022unified}.

\paragraph{Wasserstein robust classification error.}
For \(\ell=\1\), evaluating \(\mathcal R_p(\theta;\varepsilon,\1)\) requires determining how much probability mass can be transported to misclassified inputs under the Wasserstein constraint.
\citet{honguyen2023adversarial} characterize the worst-case misclassification probability through distances to misclassification under a constraint on expected transportation cost.
For a finitely supported reference distribution, their formulation reduces to a fractional knapsack problem that allocates probability mass subject to the transportation budget.
For binary linear classification, they also reformulate the distributionally robust training problem as regularized ramp loss minimization.
\citet{yang2022zerone} study WDRO with zero-one loss and relate its regularization effect to the Wasserstein radius and the density near the decision boundary.
To apply a distance-based characterization to \(\PN\), we need, for each data point \((x_i,y_i)\), the distance from \(x_i\) to the set of inputs misclassified with respect to \(y_i\).
The allocation problem is tractable once these distances are known, but computing them exactly can be difficult for nonlinear classifiers.

\paragraph{Surrogate losses and classification error.}
Classical work connects surrogate risk minimization with classification error through consistency and calibration \citep{zhang2004statistical,bartlett2006convexity,tewari2007consistency}.
For a fixed data distribution \(P\), consistency requires that any sequence of scoring functions whose surrogate risks \(\mathbb E_P[\ell(Y,f_\theta(X))]\) approach their infimum also has classification error approaching its infimum.
Both infima are taken over all measurable scoring functions.
Calibration requires that, for every conditional label distribution at a fixed input, any sequence of score vectors whose conditional expected surrogate loss approaches its infimum eventually predicts a label with maximum conditional probability.
For the restricted class of binary linear classifiers, \citet{bendavid2012minimizing} show that a convex surrogate risk minimizer can have classification error arbitrarily close to one even when another classifier in the class has arbitrarily small classification error.
For adversarial classification, \citet{bao2020calibrated}, \citet{awasthi2021calibration}, and \citet{meunier2022consistency} investigate calibration and consistency when each input can be perturbed within a prescribed neighborhood.
\citet{dobriban2023provable} establish calibration results for surrogate losses in adversarial binary linear classification under a Gaussian mixture model.
For adversarial binary classification, \citet{frank2023adversarial} characterize consistency uniformly over data distributions, while \citet{frank2026uniqueness} give a characterization for convex surrogate losses at a fixed distribution under suitable assumptions.
\citet{mianjy2023robustness} study projected gradient descent attacks based on reflected surrogate losses in adversarial training of two-layer neural networks.
\citet{robey2024adversarial} show that, for a fixed classifier, maximizing cross-entropy over admissible input perturbations need not produce a misclassified input even when one exists.
\citet{hu2018drsl} study distributionally robust classification with an \(f\)-divergence constraint and show that, under suitable assumptions, there exists a minimizer of worst-case surrogate risk over all measurable scoring functions that also minimizes classification error under the reference distribution.

Consistency alone does not imply that a hypothesis with smaller surrogate risk must have smaller classification error.
\citet{nguyen2009surrogate} and \citet{duchi2018multiclass} study universal equivalence of losses, which requires two losses to give the same ordering of input transformations (quantizers) for every data distribution, after minimizing the risk over classifiers for each quantizer.
In an analysis of symmetric label noise, \citet{vanrooyen2015learning} use the classical characterization that two losses give the same ordering of joint distributions of labels and scores by their expected losses if and only if they are related by a positive affine transformation.
Our comparison instead fixes the reference distribution \(\PN\), the Wasserstein constraint, and asks which losses \(\ell\) ensure that \(\mathcal R_p(\theta;\varepsilon,\ell)\) and \(\mathcal R_p(\theta;\varepsilon,\1)\) induce the same ordering of hypotheses.

\paragraph{Adversarial attacks.}
For a fixed hypothesis \(\theta\), we seek distributions \(Q\) satisfying the constraint \(W_p(\PN,Q)\le\varepsilon\) that give large values of \(\mathbb E_Q[\ell(Y,f_\theta(X))]\) in \eqref{eq:wdro-risk}.
Under label-preserving transportation, \(p=\infty\) bounds each input perturbation by \(\varepsilon\), as in standard adversarial robustness evaluation \citep{madry2018towards,bai2023wasserstein}.
Projected gradient descent (PGD) attacks seek to maximize a loss within the prescribed perturbation set \citep{madry2018towards}.
AutoAttack \citep{croce2020reliable} combines Auto-PGD variants using cross-entropy (CE) and Difference of Logits Ratio (DLR) losses with the Fast Adaptive Boundary (FAB) attack \citep{croce2020minimally} and Square Attack \citep{andriushchenko2020square}.
The minimum-margin attack \citep{gao2022fast} optimizes a prediction-margin objective, while BaB-Attack \citep{zhang2022branch} uses branch and bound to search for adversarial examples in the activation space of ReLU networks.
Methods that search for small adversarial perturbations include DeepFool (DF) \citep{moosavidezfooli2016deepfool}, the Carlini--Wagner (CW) attacks \citep{carlini2017towards}, Decoupled Direction and Norm (DDN) \citep{rony2019decoupling}, and Fast Minimum-Norm (FMN) attacks \citep{pintor2021fast}.
For \(p<\infty\), feasibility requires a transport plan whose expected \(p\)th power of the transportation distance is at most \(\varepsilon^p\), allowing perturbation sizes to vary across inputs.
\citet{bai2023wasserstein} develop Wasserstein gradient attacks for distributional constraints.
\citet{bertolace2026machine} develop Partial Sinkhorn for an entropically regularized formulation of attacks under a label-preserving \(W_1\) transportation budget.
Attack outputs provide candidate support points for distributions \(Q\) in \eqref{eq:wdro-risk}.
Our approach separates the search for these points from the optimization of the probability mass transported to them under the Wasserstein constraint.

\paragraph{Contributions.}
Our main contributions are as follows.
\begin{itemize}
\item We characterize the losses \(\ell\) for which \(\mathcal R_p(\theta;\varepsilon,\ell)\) and \(\mathcal R_p(\theta;\varepsilon,\1)\) give the same ordering of hypotheses, uniformly over empirical reference distributions \(\PN\), radii \(\varepsilon\), and measurable scoring functions \(f_\theta\).
Theorem~\ref{thm:losses-preserving-robust-error-order} shows that this holds exactly when the loss is a positive multiple of the classification error loss plus an offset depending only on the label.
We refer to disagreement between these orderings as \emph{surrogate loss misalignment}.
For a class of surrogate losses including CE loss, DLR loss, and other margin-based losses, Theorem~\ref{thm:surrogate-loss-misalignment} shows that, at every positive radius \(\varepsilon\), a hypothesis with robust classification error arbitrarily close to one can have a smaller WDRO risk under the surrogate loss than a hypothesis with zero robust classification error.
Both results hold for every Wasserstein exponent \(p\in[1,\infty]\).

\item We develop Wasserstein distributional attacks (WDA), a framework consisting of two stages, finding support and weight generation, summarized in Figure~\ref{fig:wda-structure}.
The first stage uses attacks to construct candidate support sets for each data point, while the second allocates probability mass among these candidates to maximize the expected loss under the Wasserstein constraint.
For finding support, we introduce Algorithm~\ref{alg:error-point-search}, which estimates distances to misclassification by searching for nearby misclassified inputs and refining those found by bisection.
The resulting WDA risk \(\widetilde{\mathcal R}_p(\theta;\varepsilon,\ell)\) provides a computable lower bound on \(\mathcal R_p(\theta;\varepsilon,\ell)\) for every \(p\in[1,\infty]\).
    
\item We benchmark WDA against Auto-PGD with CE and DLR losses (APGD-CE and APGD-DLR) and AutoAttack (AA) \citep{croce2020reliable}, the minimum-margin (MM3, abbreviated here as MM) attack \citep{gao2022fast}, and Wasserstein projected gradient descent with rectified DLR loss (W-PGD-ReDLR, abbreviated here as W-PGD) \citep{bai2023wasserstein} on CIFAR-10, CIFAR-100, and ImageNet.
At \(p=2\), WDA using W-PGD yields lower classification accuracy than Direct W-PGD and WDA using Algorithm~\ref{alg:error-point-search} achieves the lowest classification accuracy on all tested models, while having lower runtime than using W-PGD.

\item We illustrate the surrogate loss misalignment studied in Theorems~\ref{thm:losses-preserving-robust-error-order} and~\ref{thm:surrogate-loss-misalignment} on CIFAR-10 by identifying pairs of trained neural classifiers whose WDA risks give opposite orderings under classification error and CE loss.
\end{itemize}

\paragraph{Organization.}
Section~\ref{sec:notation} introduces the notation and definitions.
Section~\ref{sec:surrogate-loss-misalignment} studies the relationship between WDRO risks for surrogate losses and classification error.
Section~\ref{sec:wasserstein-distributional-attacks} presents WDA.
Section~\ref{sec:empirical-evaluation} evaluates WDA and illustrates surrogate loss misalignment.
Section~\ref{sec:conclusion} concludes.
Appendix~\ref{app:proofs} contains the proofs, and Appendix~\ref{app:experimental-details} provides additional experimental details and results.

\section{Notation}
\label{sec:notation}

We write \(\mathbb N=\{1,2,\ldots\}\) and \([m]=\{1,\ldots,m\}\) for \(m\in\mathbb N\).
For \(r\in[1,\infty]\), \(\|\cdot\|_r\) denotes the usual \(r\)-norm, and \(\langle\cdot,\cdot\rangle\) denotes the Euclidean inner product.
The vector \(e_1=(1,0,\ldots,0)\) is the first standard basis vector.
We write \(\delta_z\) for the Dirac probability measure at \(z\) and \(\mathbb E_P\) for expectation under a probability measure \(P\).
All logarithms are natural.

\paragraph{Data space and reference distribution.}
Let \(\X=\mathbb R^n\) be the input space and \(\Y=[K]\) the label space, where \(n,K\in\mathbb N\) and \(K\ge2\).
The data space is \(\Z=\X\times\Y\), with elements \(z=(x,y)\).
We equip \(\X\) with its Borel \(\sigma\)-algebra, \(\Y\) with the discrete \(\sigma\)-algebra, and \(\Z\) with the product \(\sigma\)-algebra.
Given \(N\in\mathbb N\) data points \(z_i=(x_i,y_i)\in\Z\) and weights \(\mu_i>0\), \(i\in[N]\), satisfying \(\sum_{i=1}^N\mu_i=1\), we take the weighted empirical distribution \(\PN=\sum_{i=1}^N\mu_i\delta_{z_i}\) as the reference distribution.
Uniform weights \(\mu_i=1/N\) give the usual empirical distribution.
For a probability measure \(P\) on \(\Z\), expectations are taken under \((X,Y)\sim P\).

\paragraph{Hypotheses and losses.}
A hypothesis \(\theta\) is represented by a measurable scoring function \(f_\theta:\X\to\mathbb R^K\), with \(f_\theta(x)=(f_{\theta,1}(x),\ldots,f_{\theta,K}(x))\), where \(f_{\theta,k}(x)\) is the score assigned to label \(k\).
The hypothesis class in Theorems~\ref{thm:losses-preserving-robust-error-order} and~\ref{thm:surrogate-loss-misalignment} consists of all such measurable scoring functions.
For example, in a neural classifier, \(\theta\) specifies the weights and biases, and \(f_{\theta,k}(x)\) is the output logit for label \(k\).
A measurable loss \(\ell:[K]\times\mathbb R^K\to\mathbb R\) evaluates a true label together with a score vector.
For \(z=(x,y)\in\Z\), define \(\ell_\theta(z)=\ell(y,f_\theta(x))\).
For a label \(y\in[K]\) and score vector \(s=(s_1,\ldots,s_K)\in\mathbb R^K\), the classification error loss is
\begin{equation}
\1(y,s)
=
\begin{cases}
1, & s_y\le\max_{k\ne y}s_k,\\
0, & s_y>\max_{k\ne y}s_k,
\end{cases}
\label{eq:classification-error-loss}
\end{equation}
We count score ties as errors.
For the given label \(y\), we call \(s\) a correct score when \(\1(y,s)=0\) and an error score when \(\1(y,s)=1\).
The correct classification domain under \(\theta\) is
\begin{equation}
\mathcal S_\theta
=
\{(x,y)\in\Z:
\operatorname*{argmax}_{k\in[K]}f_{\theta,k}(x)=\{y\}\}
=
\{(x,y)\in\Z:
f_{\theta,y}(x)>\max_{k\ne y}f_{\theta,k}(x)\}.
\label{eq:correct-domain}
\end{equation}
Here \(\operatorname*{argmax}\) denotes the set of maximizers.
The corresponding error domain is
\begin{equation}
\Z\setminus\mathcal S_\theta
=
\{(x,y)\in\Z:
f_{\theta,y}(x)\le\max_{k\ne y}f_{\theta,k}(x)\}.
\label{eq:error-domain}
\end{equation}
We call \(z\in\mathcal S_\theta\) a correctly classified point and \(z\in\Z\setminus\mathcal S_\theta\) an error point.
Thus \(\1(y,f_\theta(x))\) equals one on \(\Z\setminus\mathcal S_\theta\) and zero on \(\mathcal S_\theta\).

\paragraph{Ground cost and Wasserstein distance.}
Fix \(r\in[1,\infty]\) and define the ground cost on \(\Z\) by
\begin{equation}
d_r((x,y),(x',y'))
=
\begin{cases}
\|x-x'\|_r, & y=y',\\
+\infty, & y\ne y'.
\end{cases}
\label{eq:label-preserving-norm-cost}
\end{equation}
Transport with finite cost therefore preserves labels.
For probability measures \(P,Q\) on \(\Z\), let \(\Pi(P,Q)\) denote the set of couplings, that is, probability measures on \(\Z\times\Z\) with first marginal \(P\) and second marginal \(Q\).
We use the Wasserstein distances induced by \(d_r\) \citep{villani2009optimal,santambrogio2015optimal}.
Define \(W^p_p(P,Q)=\inf_{\pi\in\Pi(P,Q)}\int_{\Z\times\Z}d_r(z,z')^p\,d\pi(z,z')\) for \(1\le p<\infty\).
For \(p=\infty\), define \(W_\infty(P,Q)=\inf_{\pi\in\Pi(P,Q)}\operatorname*{ess\,sup}_{(z,z')\sim\pi}d_r(z,z')\).
These distances may take the value \(+\infty\).
Their dependence on the fixed norm order \(r\) is suppressed in the notation \(W_p\).

\paragraph{Wasserstein distributionally robust optimization.}
For \(p\in[1,\infty]\) and \(\varepsilon\ge0\), the WDRO risk in \eqref{eq:wdro-risk} is
\[
\mathcal R_p(\theta;\varepsilon,\ell)
=
\sup_{Q:\,W_p(\PN,Q)\le\varepsilon}
\mathbb E_Q\!\left[\ell(Y,f_\theta(X))\right],
\]
where \(Q\) ranges over probability measures on \(\Z\) satisfying the Wasserstein constraint and \(\int_\Z\max\{-\ell_\theta(z),0\}\,dQ(z)<\infty\) \citep{blanchet2019transport}.
The expectation may equal \(+\infty\).
For \(\ell=\1\), the value in \eqref{eq:robust-classification-error-intro} is the worst-case probability of the error domain,
\begin{equation}
\mathcal R_p(\theta;\varepsilon,\1)
=
\sup_{Q:\,W_p(\PN,Q)\le\varepsilon}
\mathbb E_Q\!\left[\1(Y,f_\theta(X))\right]
=
\sup_{Q:\,W_p(\PN,Q)\le\varepsilon}
Q(\Z\setminus\mathcal S_\theta).
\label{eq:robust-classification-error}
\end{equation}

\section{WDRO and Wasserstein Robust Classification Error}
\label{sec:surrogate-loss-misalignment}

We compare the orderings of hypotheses induced by \(\mathcal R_p(\theta;\varepsilon,\ell)\) and \(\mathcal R_p(\theta;\varepsilon,\1)\), holding the reference distribution and Wasserstein constraint fixed in each comparison.

\subsection{Losses preserving robust error order}

We begin with the case in which replacing the classification error loss by a surrogate loss preserves the ordering of hypotheses induced by exact robust classification error.
The ordering requirement is imposed uniformly over the empirical law, the Wasserstein radius, and the pair of hypotheses.
The following theorem shows that this holds if and only if the loss is a positive multiple of the classification error loss plus an offset depending only on the label.

\begin{theorem}[Surrogate loss misalignment 1]
\label{thm:losses-preserving-robust-error-order}
Fix \(p,r\in[1,\infty]\).
For a measurable loss \(\ell:[K]\times\mathbb R^K\to\mathbb R\) on score vectors, the following three statements are equivalent.
\begin{enumerate}
\item[\((A_0)\)]\label{cond:robust-order-A0}
For every empirical law \(\PN\) and every pair of hypotheses \(\theta,\theta'\),
\[
\mathbb E_{\PN}\!\left[\1(Y,f_\theta(X))\right]
>
\mathbb E_{\PN}\!\left[\1(Y,f_{\theta'}(X))\right]
\Leftrightarrow
\mathbb E_{\PN}\!\left[\ell(Y,f_\theta(X))\right]
>
\mathbb E_{\PN}\!\left[\ell(Y,f_{\theta'}(X))\right].
\]
\item[\((A_+)\)]\label{cond:robust-order-Aplus}
For every radius \(\varepsilon>0\), every empirical law \(\PN\) and every pair of hypotheses \(\theta,\theta'\),
\[
\mathcal R_p(\theta;\varepsilon,\1)
>
\mathcal R_p(\theta';\varepsilon,\1)
\Leftrightarrow
\mathcal R_p(\theta;\varepsilon,\ell)
>
\mathcal R_p(\theta';\varepsilon,\ell).
\]
\item[\((B)\)]\label{cond:robust-order-B}
There exist \(a>0\) and \(b_1,\ldots,b_K\in\mathbb R\) such that, for every hypothesis \(\theta\), every \(x\in\mathbb R^n\), and every \(y\in[K]\),
\[
\ell(y,f_\theta(x))
=
b_y
+
a\,\1(y,f_\theta(x)).
\]
\end{enumerate}
\end{theorem}

The proof is deferred to Appendix~\ref{app:proof-losses-preserving-robust-error-order}.
Under condition \hyperref[cond:robust-order-B]{\((B)\)}, label preservation gives
\[
\mathcal R_p(\theta;\varepsilon,\ell)
=
\sum_{y=1}^K b_y\,\PN(Y=y)
+
a\,\mathcal R_p(\theta;\varepsilon,\1).
\]
The first term is independent of \(\theta\), and \(a>0\), so the two objectives give the same ordering of hypotheses, including ties.
In fact, the proof shows that \hyperref[cond:robust-order-B]{\((B)\)} follows from the forward implication (\(\Rightarrow\)) alone in either \hyperref[cond:robust-order-A0]{\((A_0)\)} or \hyperref[cond:robust-order-Aplus]{\((A_+)\)}.

Condition \hyperref[cond:robust-order-B]{\((B)\)} requires the loss to be constant on correct scores for each label.
This excludes CE loss (\(\ell_{\rm CE}\)), commonly used in adversarial training \citep{goodfellow2015explaining,madry2018towards}; logistic, exponential, and hinge losses studied in classification surrogate theory \citep{bartlett2006convexity}; and DLR loss (\(\ell_{\rm DLR}\)) used in robustness evaluation \citep{croce2020reliable}.
We illustrate this restriction in the binary case \(K=2\).
Fix \(y=1\) and consider the constant scoring functions \(f_{\theta_t}(x)\equiv(t,0)\) for \(t>0\).
These scores are correct for every \(t>0\), so \hyperref[cond:robust-order-B]{\((B)\)} would require \(\ell(1,f_{\theta_t}(x))\) to be independent of \(t\).
However,
\[
\ell_{\rm CE}(1,f_{\theta_t}(x))=\log(1+e^{-t}),
\qquad
\ell_{\rm margin}(1,f_{\theta_t}(x))
=f_{\theta_t,2}(x)-f_{\theta_t,1}(x)
=-t.
\]
Both losses vary with \(t\), so neither satisfies \hyperref[cond:robust-order-B]{\((B)\)}.

\subsection{WDRO with surrogate losses can prefer high robust error}

We now ask how much the robust classification errors of two hypotheses can differ when their WDRO risks under surrogate loss give the opposite ordering.
Consider again the binary case \(K=2\), with label \(y=1\), margin loss (\(\ell_{\rm margin}\)), and constant scoring functions \(f_{\theta_t}(x)\equiv(t,0)\), now allowing \(t\ge0\).
For \(t>0\), these scores are correct and have margin loss \(-t\), which approaches \(0\) as \(t\downarrow0\).
The tied score \(f_{\theta_0}(x)=(0,0)\) is an error score and has margin loss \(0\).
Since the margin loss is negative on all correct scores and nonnegative on all error scores, we obtain
\[
\inf_{\theta:\,(x,1)\in\mathcal S_\theta}\ell_{\rm margin}(1,f_\theta(x))
<
\inf_{\theta:\,(x,1)\in\Z\setminus\mathcal S_\theta}\ell_{\rm margin}(1,f_\theta(x))
=
\sup_{\theta:\,(x,1)\in\mathcal S_\theta}\ell_{\rm margin}(1,f_\theta(x))
=
0.
\]
The theorem below takes this loss comparison as its assumption and shows that WDRO with a surrogate loss can prefer a hypothesis with nearly maximal robust classification error over a hypothesis with zero robust classification error.

\begin{theorem}[Surrogate loss misalignment 2]
\label{thm:surrogate-loss-misalignment}
Fix \(p,r\in[1,\infty]\).
Let \(\ell:[K]\times\mathbb R^K\to\mathbb R\) be a measurable loss on score vectors.
Assume that the following condition holds.
\begin{enumerate}
\item[\((M)\)]\label{cond:loss-misalignment}
There exist \(x\in\mathbb R^n\) and \(y\in[K]\) such that
\[
\inf_{\theta:\,(x,y)\in\mathcal S_\theta}\ell(y,f_\theta(x))
<
\inf_{\theta:\,(x,y)\in\Z\setminus\mathcal S_\theta}\ell(y,f_\theta(x))
\le
\sup_{\theta:\,(x,y)\in\mathcal S_\theta}\ell(y,f_\theta(x)).
\]
\end{enumerate}
Then, for every \(\varepsilon>0\) and every \(\zeta\in(0,1)\), there exist an empirical law \(\PN\) and hypotheses \(\theta_1,\theta_2\) such that \(\mathcal R_p(\theta_2;\varepsilon,\1)=0\), and
\[
\mathcal R_p(\theta_1;\varepsilon,\ell)<\mathcal R_p(\theta_2;\varepsilon,\ell),
\qquad
\mathcal R_p(\theta_1;\varepsilon,\1)-\mathcal R_p(\theta_2;\varepsilon,\1)
>
1-\zeta.
\]
\end{theorem}

The proof is deferred to Appendix~\ref{app:proof-surrogate-loss-misalignment}.
Condition \hyperref[cond:loss-misalignment]{\((M)\)} is incompatible with \hyperref[cond:robust-order-B]{\((B)\)}.
Under \hyperref[cond:robust-order-B]{\((B)\)}, all correct scores for label \(y\) have loss \(b_y\), while all error scores have loss \(b_y+a>b_y\), contradicting the right inequality in \hyperref[cond:loss-misalignment]{\((M)\)}.
Failure of \hyperref[cond:robust-order-B]{\((B)\)} alone does not imply \hyperref[cond:loss-misalignment]{\((M)\)}.

\begin{example}[Common losses covered by Theorem~\ref{thm:surrogate-loss-misalignment}]
We verify condition \hyperref[cond:loss-misalignment]{\((M)\)} for the following surrogate losses.
Throughout, fix \((x,y)\in\mathbb R^n\times[K]\) and use constant scoring functions to realize the desired score vectors.

\begin{enumerate}
\item \textbf{Cross-entropy loss:}
\[
\ell_{\rm CE}(y,f_\theta(x))
=
\log\sum_{k=1}^K e^{f_{\theta,k}(x)}-f_{\theta,y}(x)
=
\log\left(1+\sum_{k\ne y}e^{f_{\theta,k}(x)-f_{\theta,y}(x)}\right).
\]
\begin{enumerate}
\item For correct scores, \(f_{\theta,k}(x)-f_{\theta,y}(x)<0\) for every \(k\ne y\).
Taking \(f_{\theta,y}(x)=t\) and \(f_{\theta,k}(x)=0\) for \(k\ne y\), with \(t>0\), gives \(\ell_{\rm CE}(y,f_\theta(x))=\log(1+(K-1)e^{-t})\).
Letting \(t\to\infty\) gives the infimum \(0\), while letting \(t\downarrow0\) gives the supremum \(\log K\).
\item For error scores, some \(i\ne y\) satisfies \(f_{\theta,i}(x)\ge f_{\theta,y}(x)\).
Hence, \(\ell_{\rm CE}(y,f_\theta(x))\ge\log\left(1+e^{f_{\theta,i}(x)-f_{\theta,y}(x)}\right)\ge\log2\).
Taking \(f_{\theta,y}(x)=f_{\theta,i}(x)=0\) and \(f_{\theta,k}(x)=-t\) for all \(k\notin\{y,i\}\), then letting \(t\to\infty\), this lower bound is approached.
\end{enumerate}
Since \(0<\log2\le\log K\), \hyperref[cond:loss-misalignment]{\((M)\)} holds with
\[
\inf_{\theta:\,(x,y)\in\mathcal S_\theta}\ell_{\rm CE}(y,f_\theta(x))
<
\inf_{\theta:\,(x,y)\in\Z\setminus\mathcal S_\theta}\ell_{\rm CE}(y,f_\theta(x))
\le
\sup_{\theta:\,(x,y)\in\mathcal S_\theta}\ell_{\rm CE}(y,f_\theta(x)).
\]

\item \textbf{Margin loss:}
\[
\ell_{\rm margin}(y,f_\theta(x))
=
\max_{k\ne y}f_{\theta,k}(x)-f_{\theta,y}(x).
\]
\begin{enumerate}
\item For correct scores, \(\ell_{\rm margin}(y,f_\theta(x))<0\).
Taking \(f_{\theta,y}(x)=t\) and \(f_{\theta,k}(x)=0\) for \(k\ne y\), with \(t>0\), gives \(\ell_{\rm margin}(y,f_\theta(x))=-t\).
Letting \(t\to\infty\) gives the infimum \(-\infty\) over correct scores, while letting \(t\downarrow0\) gives the supremum \(0\).
\item For error scores, \(\ell_{\rm margin}(y,f_\theta(x))\ge0\).
The score with \(f_{\theta,k}(x)=0\) for every \(k\) is an error score and attains loss \(0\).
\end{enumerate}
Therefore \hyperref[cond:loss-misalignment]{\((M)\)} holds with
\[
\inf_{\theta:\,(x,y)\in\mathcal S_\theta}\ell_{\rm margin}(y,f_\theta(x))
<
\inf_{\theta:\,(x,y)\in\Z\setminus\mathcal S_\theta}\ell_{\rm margin}(y,f_\theta(x))
=
\sup_{\theta:\,(x,y)\in\mathcal S_\theta}\ell_{\rm margin}(y,f_\theta(x)).
\]
\item \textbf{Difference of logits ratio loss:}
\[
\ell_{\rm DLR}(y,f_\theta(x))
=
-
\frac{f_{\theta,y}(x)-\max_{k\ne y}f_{\theta,k}(x)}{f_{\theta,(1)}(x)-f_{\theta,(3)}(x)},
\]
where \(K\ge3\), and \(f_{\theta,(1)}(x)\ge f_{\theta,(2)}(x)\ge f_{\theta,(3)}(x)\) are the three largest logits.
We set the loss to \(0\) when the denominator is zero.
\begin{enumerate}
\item For correct scores, \(f_{\theta,y}(x)-\max_{k\ne y}f_{\theta,k}(x)>0\), so \(\ell_{\rm DLR}(y,f_\theta(x))<0\).
Choose \(i\ne y\), set \(f_{\theta,y}(x)=t\), \(f_{\theta,i}(x)=0\), \(f_{\theta,k}(x)=-1\) for \(k\ne y\) and \(k\ne i\).
Then \((x,y)\in\mathcal S_{\theta}\) for \(t>0\), and \(\ell_{\rm DLR}(y,f_{\theta}(x))=-\frac{t}{1+t}\uparrow0\) as \(t\downarrow0\).
\item  For error scores, \(f_{\theta,y}(x)-\max_{k\ne y}f_{\theta,k}(x)\le0\) so \(\ell_{\rm DLR}(y,f_\theta(x))\ge0\). Moreover, having \(f_{\theta,(1)}(x)=f_{\theta,(3)}(x)\) or \(f_{\theta,y}(x)=\max_{k\ne y}f_{\theta,k}(x)\) attain loss \(0\).
\end{enumerate}
Therefore \hyperref[cond:loss-misalignment]{\((M)\)} holds with
\[
\inf_{\theta:\,(x,y)\in\mathcal S_\theta}\ell_{\rm DLR}(y,f_\theta(x))
<
\inf_{\theta:\,(x,y)\in\Z\setminus\mathcal S_\theta}\ell_{\rm DLR}(y,f_\theta(x))
=
\sup_{\theta:\,(x,y)\in\mathcal S_\theta}\ell_{\rm DLR}(y,f_\theta(x)).
\]
\item More generally, margin-based losses:
\[
\ell_\phi(y,f_\theta(x))=\phi\left(f_{\theta,y}(x)-\max_{k\ne y}f_{\theta,k}(x)\right),
\]
where \(\phi:\mathbb R\to\mathbb R\) is nonincreasing, right-continuous at \(0\) and satisfies \(\inf_{u>0}\phi(u)<\phi(0)\).
The margin \(f_{\theta,y}(x)-\max_{k\ne y}f_{\theta,k}(x)\) ranges over \((0,\infty)\) for correct scores and \((-\infty,0]\) for error scores.
Since \(\phi\) is nonincreasing, \(\inf_{u\le0}\phi(u)=\phi(0)\).
Its right-continuity at \(0\) further gives \(\sup_{u>0}\phi(u)=\phi(0)\).
Therefore \hyperref[cond:loss-misalignment]{\((M)\)} holds with
\[
\inf_{\theta:\,(x,y)\in\mathcal S_\theta}\ell_\phi(y,f_\theta(x))
<
\phi(0)
=
\inf_{\theta:\,(x,y)\in\Z\setminus\mathcal S_\theta}\ell_\phi(y,f_\theta(x))
=
\sup_{\theta:\,(x,y)\in\mathcal S_\theta}\ell_\phi(y,f_\theta(x)).
\]
In binary classification, the choices \(\phi(u)=\log(1+e^{-u})\), \(\phi(u)=e^{-u}\), and \(\phi(u)=\max\{0,1-u\}\) give the \textbf{logistic}, \textbf{exponential}, and \textbf{hinge} losses, respectively.
\end{enumerate}
\end{example}

\section{Wasserstein Distributional Attacks}
\label{sec:wasserstein-distributional-attacks}

We construct a feasible adversarial distribution in two stages, finding support and weight generation.
For each data point \(z_i=(x_i,y_i)\), the first stage constructs a finite candidate support set \(\widetilde{\Z}_i\subseteq\X\times\{y_i\}\) containing \(z_i\).
The second stage allocates the probability mass \(\mu_i\) of each data point among its candidates to maximize the expected loss subject to the Wasserstein constraint.
For fixed candidate support sets, we call the optimal value of the second stage the WDA risk, denoted by \(\widetilde{\mathcal R}_p(\theta;\varepsilon,\ell)\).
An optimal solution induces an adversarial distribution \(Q_{\rm adv}\) satisfying
\begin{equation}
W_p(\PN,Q_{\rm adv})\le\varepsilon,
\qquad
\mathbb E_{Q_{\rm adv}}\!\left[\ell(Y,f_\theta(X))\right]
=
\widetilde{\mathcal R}_p(\theta;\varepsilon,\ell)
\le
\mathcal R_p(\theta;\varepsilon,\ell).
\label{eq:wda-feasibility-and-lower-bound}
\end{equation}
The tightness of this lower bound depends on the candidate support sets constructed in the first stage.

We use \emph{Direct} to denote the procedure that retains one returned attack point for each data point and assigns its full probability mass to that destination.
Figure~\ref{fig:wda-structure} summarizes WDA and contrasts it with Direct.

\begin{figure}[tbp]
\centering

\definecolor{Cerulean}{RGB}{0,114,178}
\definecolor{Vermillion}{RGB}{213,94,0}
\colorlet{SupportFill}{Cerulean!30}
\colorlet{SupportRule}{Cerulean}
\colorlet{WeightFill}{Vermillion!30}
\colorlet{WeightRule}{Vermillion}
\colorlet{NeutralFill}{black!15}
\colorlet{NeutralRule}{black}
\colorlet{FigureInk}{black}
\colorlet{FigureMuted}{black!60}

\begin{tikzpicture}[
blocktitle/.style={
    font=\small\bfseries,
    align=center,
    text width=14.18cm,
    inner xsep=0.30cm,
    inner ysep=0.16cm,
    outer sep=0pt,
    line width=0.55pt
},
blockbody/.style={
    font=\small,
    align=left,
    text width=14.18cm,
    inner xsep=0.30cm,
    inner ysep=0.25cm,
    outer sep=0pt,
    line width=0.55pt,
    text=FigureInk
},
neutralTitle/.style={blocktitle,draw=NeutralRule,fill=NeutralFill},
neutralBody/.style={blockbody,draw=NeutralRule},
supportTitle/.style={blocktitle,draw=SupportRule,fill=SupportFill},
supportBody/.style={blockbody,draw=SupportRule},
weightTitle/.style={blocktitle,draw=WeightRule,fill=WeightFill},
weightBody/.style={blockbody,draw=WeightRule},
flow/.style={-{Stealth[length=4.5pt,width=5.5pt]},line width=0.65pt}
]

\node[neutralTitle] (inputTitle) {Input};

\node[neutralBody,anchor=north] (inputBody) at (inputTitle.south)
{%
\renewcommand{\arraystretch}{1.28}
\begin{tabular}{@{}>{\bfseries\raggedright\arraybackslash}p{3.45cm}>{\raggedright\arraybackslash}p{10.17cm}@{}}
Empirical law
& \(z_i=(x_i,y_i)\), \(\mu_i>0\), \(\sum_{i=1}^N\mu_i=1\), \(\PN=\sum_{i=1}^N\mu_i\delta_{z_i}\)\\
Model and loss
& \(f_\theta\) and \(\ell\)\\
WDRO problem
& \(\mathcal R_p(\theta;\varepsilon,\ell)\) in \eqref{eq:wdro-risk} and \(d_r\) in \eqref{eq:label-preserving-norm-cost}\\
Attacks
& Algorithm~\ref{alg:error-point-search} / APGD / BaB-Attack / CW / DDN / FAB / FMN / MM / PGD / Square Attack / W-PGD / any other attack
\end{tabular}
};

\node[supportTitle,anchor=north] (supportTitleNode)
at ([yshift=-0.48cm]inputBody.south)
{Finding support};

\node[supportBody,anchor=north] (supportBodyNode)
at (supportTitleNode.south)
{%
\renewcommand{\arraystretch}{1.24}
\begin{tabular}{@{}>{\bfseries\raggedright\arraybackslash}p{3.45cm}>{\raggedright\arraybackslash}p{10.17cm}@{}}
Initialize
& \(\widetilde{\Z}_i\gets\{z_i\}\) for every \(i\in[N]\)\\

\(\bullet\)\enspace Option 1. WDA
& For each \(i\), collect all \(x_{\rm adv}\) returned for \(z_i\) by the selected attack and update according to \eqref{eq:candidate-support-construction}
\par\smallskip
(If attack is not Algorithm~\ref{alg:error-point-search}, also add the error points obtained by bisection according to \eqref{eq:candidate-support-refinement})
\\

\textcolor{FigureMuted}{\(\bullet\)\enspace Option 2. Direct}
& \textcolor{FigureMuted}{For each \(i\), let \(\widetilde z_i=(x_{\rm adv},y_i)\) be the single point returned by the selected attack and set \(\widetilde{\Z}_i\gets\{z_i,\widetilde z_i\}\)}
\\

Return
& Candidate sets \(\widetilde{\Z}=(\widetilde{\Z}_i)_{i=1}^N\)
\end{tabular}
};

\node[weightTitle,anchor=north] (weightTitleNode)
at ([yshift=-0.48cm]supportBodyNode.south)
{Weight generation};

\node[weightBody,anchor=north] (weightBodyNode)
at (weightTitleNode.south)
{%
\renewcommand{\arraystretch}{1.28}
\begin{tabular}{@{}>{\bfseries\raggedright\arraybackslash}p{3.45cm}>{\raggedright\arraybackslash}p{10.17cm}@{}}
Transport weights
& For \(\widetilde z\in\widetilde{\Z}_i\), \(w_{i,\widetilde z}\ge0\) is probability mass transported from \(z_i\) to \(\widetilde z\).
For each \(i\), the transport weights satisfy \(\sum_{\widetilde z\in\widetilde{\Z}_i}w_{i,\widetilde z}=\mu_i\)
\\

\(\bullet\)\enspace Option 1. WDA
& Solve \eqref{eq:wda-risk}, replacing \eqref{eq:wda-budget} with \eqref{eq:wda-risk-infinity} when \(p=\infty\), to obtain optimal transport weights \(w=(w_{i,\widetilde z})\)
\\

\textcolor{FigureMuted}{\(\bullet\)\enspace Option 2. Direct}
& \textcolor{FigureMuted}{
\(w_{i,\widetilde z}=\mu_i\) if \(\widetilde z=\widetilde z_i\), and \(0\) otherwise
}
\\

Return
& Transport weights \(w\)
\end{tabular}
};

\node[neutralTitle,anchor=north] (outputTitle)
at ([yshift=-0.48cm]weightBodyNode.south)
{Output};

\node[neutralBody,anchor=north,align=center] (outputBody)
at (outputTitle.south)
{%
\parbox{10cm}{%
\centering
\(\displaystyle
Q_{\rm adv}
:=
\sum_{i=1}^N
\sum_{\widetilde z\in\widetilde{\Z}_i}
w_{i,\widetilde z}\delta_{\widetilde z}
\quad \text{\small (at most \(N+1\) support points)}\)}
};

\draw[flow] (inputBody.south) -- (supportTitleNode.north);
\draw[flow] (supportBodyNode.south) -- (weightTitleNode.north);
\draw[flow] (weightBodyNode.south) -- (outputTitle.north);

\end{tikzpicture}

\caption{Construction of \(Q_{\rm adv}\) by WDA and Direct.
Direct (Option 2) assigns the full probability mass of each point to one returned attack point, as in existing attack protocols such as AutoAttack \citep{croce2020reliable}, MM \citep{gao2022fast}, and W-PGD \citep{bai2023wasserstein}.
Our WDA (Option 1) instead collects candidate points, adds bisection refinements according to \eqref{eq:candidate-support-refinement}, and optimizes the transport weights jointly over the resulting sets subject to the transport constraint.}
\label{fig:wda-structure}
\end{figure}

\subsection{Finding support}
\label{subsec:finding-support}

For each data point \(z_i=(x_i,y_i)\), the selected attacks return a finite collection of inputs \(x_{\rm adv}\in\X\).
We pair these inputs with the original label \(y_i\) and construct the candidate support set according to
\begin{equation}
\label{eq:candidate-support-construction}
\begin{aligned}
\widetilde{\Z}_i
&\gets
\{z_i\},
\\
\widetilde{\Z}_i
&\gets
\widetilde{\Z}_i
\cup
\left\{
(x_{\rm adv},y_i):
x_{\rm adv}\text{ is returned for }z_i
\right\},
\end{aligned}
\qquad
i\in[N].
\end{equation}

\begin{algorithm}[tb]
\caption{Nearby error point search}
\label{alg:error-point-search}
\begin{algorithmic}[1]
\Require \(f_\theta\), \(z_i=(x_i,y_i)\in\mathcal S_\theta\), \(r\), \(k_{\rm top}\in[K-1]\), \(\eta>0\), and \(T_{\rm step},T_{\rm bis}\in\mathbb N\)
\State choose \(\mathcal J_i\subseteq[K]\setminus\{y_i\}\), with \(|\mathcal J_i|=k_{\rm top}\), indexing the \(k_{\rm top}\) largest values among \(\{f_{\theta,j}(x_i):j\ne y_i\}\)
\State \(\widehat x_i\gets\varnothing\)
\For{\(j\in\mathcal J_i\)} \Comment{one path per screened rival label}
    \State \(x_{ij}^{(0)}\gets x_i\)
    \For{\(t=0,\ldots,T_{\rm step}-1\)}
        \State \(u_{ij}(x_{ij}^{(t)})
        \gets
        \operatorname*{argmax}_{\|u\|_r\le1}
        \langle\nabla_x (f_{\theta,j}(x_{ij}^{(t)})-f_{\theta,y_i}(x_{ij}^{(t)})),u\rangle\)
        \State \(x_{ij}^{(t+1)}
        \gets
        x_{ij}^{(t)}+\eta u_{ij}(x_{ij}^{(t)})\)
        \If{\((x_{ij}^{(t+1)},y_i)\in\Z\setminus\mathcal S_\theta\)} \Comment{detect the first entry into the error domain}
            \State \(x_{\rm cor}\gets x_{ij}^{(t)}\)
            \State \(x_{\rm err}\gets x_{ij}^{(t+1)}\)
            \For{\(s=1,\ldots,T_{\rm bis}\)} \Comment{refine the boundary by bisection}
                \State \(x_{\rm mid}\gets(x_{\rm cor}+x_{\rm err})/2\)
                \If{\((x_{\rm mid},y_i)\in\Z\setminus\mathcal S_\theta\)}
                    \State \(x_{\rm err}\gets x_{\rm mid}\)
                \Else
                    \State \(x_{\rm cor}\gets x_{\rm mid}\)
                \EndIf
            \EndFor
            \If{\(\widehat x_i=\varnothing\) or
            \(\|x_{\rm err}-x_i\|_r<\|\widehat x_i-x_i\|_r\)}
                \State \(\widehat x_i\gets x_{\rm err}\) \Comment{retain the closest point found}
            \EndIf
            \State \textbf{break} the \(t\)-loop
        \EndIf
    \EndFor
\EndFor
\If{\(\widehat x_i\ne\varnothing\)} \Comment{return a candidate if one was found}
    \State \(x_{\rm adv}\gets\widehat x_i\)
    \State \Return \(\{(x_{\rm adv},y_i)\}\)
\EndIf
\State \Return \(\varnothing\)
\end{algorithmic}
\end{algorithm}

Algorithm~\ref{alg:error-point-search} searches for a nearby error point for each correctly classified data point and uses bisection to refine the error points found.
The search is motivated by characterizations of Wasserstein robust classification error through the transport required to reach the error domain \(\Z\setminus\mathcal S_\theta\) \citep{honguyen2023adversarial,yang2022zerone}.

For WDA, we also refine error points returned by attacks other than Algorithm~\ref{alg:error-point-search}.
For \(z_i\in\mathcal S_\theta\) and a returned point \((x_{\rm adv},y_i)\in\Z\setminus\mathcal S_\theta\), we apply the bisection steps of Algorithm~\ref{alg:error-point-search}, initialized at \((x_{\rm cor},x_{\rm err})=(x_i,x_{\rm adv})\).
Let \(x_{\rm bis}\) denote the error endpoint after \(T_{\rm bis}\) iterations.
We augment the candidate set by
\begin{equation}
\label{eq:candidate-support-refinement}
\widetilde{\Z}_i
\gets
\widetilde{\Z}_i
\cup
\{(x_{\rm bis},y_i)\}.
\end{equation}
We write \(\widetilde{\Z}:=(\widetilde{\Z}_i)_{i=1}^N\) for the resulting collection of candidate support sets.

\subsection{Weight generation}
\label{subsec:weight-generation}

Fix \(\widetilde{\Z}\) and a loss \(\ell\).
For each \(i\in[N]\) and \(\widetilde z\in\widetilde{\Z}_i\), let \(w_{i,\widetilde z}\) denote the probability mass transported from \(z_i\) to \(\widetilde z\).
We suppress the dependence of \(\widetilde{\mathcal R}_p(\theta;\varepsilon,\ell)\) on \(\widetilde{\Z}\) in the notation.

For \(1\le p<\infty\), the WDA risk is the optimal value of the linear program
\begin{subequations}
\label{eq:wda-risk}
\begin{align}
\widetilde{\mathcal R}_p(\theta;\varepsilon,\ell)
:=
\max_{(w_{i,\widetilde z})}
\quad
&\sum_{i=1}^N
\sum_{\widetilde z\in\widetilde{\Z}_i}
w_{i,\widetilde z}\ell_\theta(\widetilde z)
\label{eq:wda-objective}
\\
\text{subject to}\quad
&w_{i,\widetilde z}\ge0
\quad
(\widetilde z\in\widetilde{\Z}_i),
\qquad
{\textstyle
\sum_{\widetilde z\in\widetilde{\Z}_i}
w_{i,\widetilde z}}
=
\mu_i,
\quad
i\in[N],
\label{eq:wda-source-weights}
\\
&{\textstyle
\sum_{i=1}^N
\sum_{\widetilde z\in\widetilde{\Z}_i}
d_r(z_i,\widetilde z)^p
w_{i,\widetilde z}}
\le
\varepsilon^p.
\label{eq:wda-budget}
\end{align}
\end{subequations}

For \(p=\infty\), \(\widetilde{\mathcal R}_\infty(\theta;\varepsilon,\ell)\) is defined by the optimization problem in \eqref{eq:wda-risk} with the shared transport constraint \eqref{eq:wda-budget} replaced by
\begin{equation}
\label{eq:wda-risk-infinity}
w_{i,\widetilde z}
=
0
\quad
\text{whenever }
d_r(z_i,\widetilde z)>\varepsilon,
\qquad
i\in[N],
\quad
\widetilde z\in\widetilde{\Z}_i.
\end{equation}
The resulting problem separates across the data points and an optimal solution is obtained by assigning the full probability mass \(\mu_i\) to a candidate maximizing \(\ell_\theta\) among those within distance \(\varepsilon\) of \(z_i\), for each \(i\in[N]\).
Consequently,
\begin{equation}
\label{eq:wda-risk-infinity-closed-form}
\widetilde{\mathcal R}_\infty(\theta;\varepsilon,\ell)
=
\sum_{i=1}^N
\mu_i
\max
\left\{
\ell_\theta(\widetilde z):
\widetilde z\in\widetilde{\Z}_i,\,
d_r(z_i,\widetilde z)\le\varepsilon
\right\}.
\end{equation}

At \(p=\infty\), WDA therefore admits an optimal solution with the same form of mass assignment as Direct in Figure~\ref{fig:wda-structure}.
For every \(p\in[1,\infty]\), under the same transport constraint, WDA attains at least as large an expected loss as Direct, provided that each Direct destination belongs to the corresponding WDA candidate set.
At \(p=\infty\), equality holds when each Direct destination attains the corresponding maximum in \eqref{eq:wda-risk-infinity-closed-form}.

For a chosen \(p\), let \(w^\star=(w_{i,\widetilde z}^\star)\) be a basic optimal solution of the corresponding problem.
The associated WDA distribution is the second marginal of the transport plan induced by \(w^\star\), given by \(Q_{\rm adv}=\sum_{i=1}^N\sum_{\widetilde z\in\widetilde{\Z}_i}w_{i,\widetilde z}^\star\delta_{\widetilde z}\).
For \(1\le p<\infty\), the \(N\) constraints \eqref{eq:wda-source-weights} and the single constraint \eqref{eq:wda-budget} imply that \(w^\star\) has at most \(N+1\) positive transport weights, so \(Q_{\rm adv}\) has at most \(N+1\) support points \citep{gao2023distributionally}.
For \(p=\infty\), the assignment given by \eqref{eq:wda-risk-infinity-closed-form} gives at most \(N\) support points.
The transport constraints \eqref{eq:wda-budget} and \eqref{eq:wda-risk-infinity} imply \(W_p(\PN,Q_{\rm adv})\le\varepsilon\), and the objective equals \(\mathbb E_{Q_{\rm adv}}[\ell(Y,f_\theta(X))]\), which establishes \eqref{eq:wda-feasibility-and-lower-bound}.

\section{Empirical Evaluation}
\label{sec:empirical-evaluation}

We compare WDA with Direct, evaluate attacks for finding support within WDA, and illustrate surrogate loss misalignment with pretrained neural networks.

Our experiments use 10 CIFAR-10 models, 4 CIFAR-100 models, and 4 ImageNet models.
We evaluate all 10,000 test images from each CIFAR dataset and the 5,000 ImageNet validation images specified by the RobustBench protocol \citep{croce2021robustbench}.
Throughout this section, we use ground norm \(r=\infty\), Wasserstein exponents \(p\in\{1,2,\infty\}\), and radius \(\varepsilon=8/255\) for CIFAR-10 and CIFAR-100 and \(\varepsilon=4/255\) for ImageNet.
For Algorithm~\ref{alg:error-point-search}, we use \(\eta=\varepsilon/4\), \(T_{\rm step}=16\), and \(T_{\rm bis}=8\), with \(k_{\rm top}=3,5,10\) for CIFAR-10, CIFAR-100, and ImageNet, respectively.
For the comparison attacks, we use the default settings of their original implementations.
All experiments were run on one NVIDIA H100 GPU.
Additional results and details appear in Appendix~\ref{app:experimental-details}.

\subsection{Numerical evaluation of WDA}
\label{subsec:wda-numerical-evaluation}

Table~\ref{tab:wda-main-results} compares our WDA against Direct on one ResNet-18 model from each dataset.
We use APGD-CE, APGD-DLR, and AA \citep{croce2020reliable}, MM \citep{gao2022fast}, and W-PGD \citep{bai2023wasserstein} as comparison attacks.
For the comparison attacks, WDA support sets contain the original data points, attack outputs at budgets \(\varepsilon\) and \(4\varepsilon\), and eligible error points obtained by 12 bisection steps according to \eqref{eq:candidate-support-refinement}, whereas Direct support sets only contain the original data points and attack outputs at budgets \(\varepsilon\).
For each attack, we keep the candidate support sets fixed across Wasserstein exponents and maximize classification error at radius \(\varepsilon\), using \eqref{eq:wda-risk} for \(p=1,2\) and \eqref{eq:wda-risk-infinity-closed-form} for \(p=\infty\).
The WDA and Direct entries report classification accuracy under the corresponding \(Q_{\rm adv}\).

\begin{table}[tb]
\centering
\fontsize{7}{9}\selectfont

\caption{Comparison of WDA and Direct on one ResNet-18 model from each dataset.
Entries report classification accuracy (\%) and runtime in minutes as WDA / Direct (runtime).
A dash indicates that no corresponding result is available.}
\label{tab:wda-main-results}

\setlength{\tabcolsep}{0.6pt}
\renewcommand{\arraystretch}{1.25}

\providecommand{\wdapair}[2]{}
\renewcommand{\wdapair}[2]{%
  \makebox[20.7pt][c]{#1}%
  \hspace{0.7pt}/\hspace{0.7pt}%
  \makebox[20.7pt][c]{#2}%
}

\providecommand{\wdacell}[4][19.2pt]{}
\renewcommand{\wdacell}[4][19.2pt]{%
  \mbox{%
    \wdapair{#2}{#3}%
    \hspace{1.5pt}%
    \makebox[#1][l]{%
      \fontsize{6}{7}\rmfamily\mdseries\upshape\selectfont #4%
    }%
  }%
}

\providecommand{\wdaname}[2][19.2pt]{}
\renewcommand{\wdaname}[2][19.2pt]{%
  \begingroup
  \setbox0=\hbox{\wdapair{}{}}%
  \mbox{%
    \makebox[\wd0][c]{#2}%
    \hspace{1.5pt}%
    \makebox[#1]{}%
  }%
  \endgroup
}

\providecommand{\wdahead}[1][19.2pt]{}
\renewcommand{\wdahead}[1][19.2pt]{%
  \wdacell[#1]{\scalebox{0.95}{WDA}}{\scalebox{0.95}{Direct}}{}%
}

\begin{tabular*}{\linewidth}{@{}c@{\hspace{4pt}\extracolsep{\fill}}*{6}{c}@{}}
\toprule
\(p\)
& \wdaname{Alg.~\ref{alg:error-point-search}}
& \wdaname{APGD-CE}
& \wdaname{APGD-DLR}
& \wdaname{MM}
& \wdaname{W-PGD}
& \wdaname[22.8pt]{AA}
\\
& \wdahead
& \wdahead
& \wdahead
& \wdahead
& \wdahead
& \wdahead[22.8pt]
\\

\midrule
\multicolumn{7}{c}{CIFAR-10, \(r=\infty,\ \varepsilon=8/255\), clean accuracy \(=84.59\)}
\\
\midrule

\(1\)
& \wdacell{\textbf{12.89}}{\textendash}{(4m)}
& \wdacell{19.97}{\textendash}{(5m)}
& \wdacell{24.80}{\textendash}{(5m)}
& \wdacell{\underline{17.31}}{\textendash}{(2m)}
& \wdacell{18.39}{\textendash}{(5m)}
& \wdacell[22.8pt]{18.93}{\textendash}{(104m)}
\\[2pt]

\(2\)
& \wdacell{\textbf{25.19}}{\textendash}{(4m)}
& \wdacell{36.42}{\textendash}{(5m)}
& \wdacell{40.59}{\textendash}{(5m)}
& \wdacell{32.74}{\textendash}{(2m)}
& \wdacell{\underline{31.42}}{46.71}{(5m)}
& \wdacell[22.8pt]{34.96}{\textendash}{(104m)}
\\[2pt]

\(\infty\)
& \wdacell{58.54}{58.60}{(4m)}
& \wdacell{58.39}{58.39}{(5m)}
& \wdacell{57.58}{57.61}{(5m)}
& \wdacell{\underline{56.01}}{\underline{56.01}}{(2m)}
& \wdacell{62.28}{\textendash}{(5m)}
& \wdacell[22.8pt]{\textbf{55.54}}{\textbf{55.54}}{(104m)}
\\[2pt]

\midrule
\multicolumn{7}{c}{CIFAR-100, \(r=\infty,\ \varepsilon=8/255\), clean accuracy \(=65.45\)}
\\
\midrule

\(1\)
& \wdacell{2.03}{\textendash}{(4m)}
& \wdacell{1.25}{\textendash}{(5m)}
& \wdacell{1.11}{\textendash}{(5m)}
& \wdacell{0.99}{\textendash}{(2m)}
& \wdacell{\textbf{0.00}}{\textendash}{(4m)}
& \wdacell[22.8pt]{\underline{0.74}}{\textendash}{(45m)}
\\[2pt]

\(2\)
& \wdacell{\textbf{4.09}}{\textendash}{(4m)}
& \wdacell{13.12}{\textendash}{(5m)}
& \wdacell{14.52}{\textendash}{(5m)}
& \wdacell{\underline{6.38}}{\textendash}{(2m)}
& \wdacell{8.14}{14.52}{(4m)}
& \wdacell[22.8pt]{10.80}{\textendash}{(45m)}
\\[2pt]

\(\infty\)
& \wdacell{29.94}{29.99}{(4m)}
& \wdacell{33.55}{33.55}{(5m)}
& \wdacell{28.85}{28.85}{(5m)}
& \wdacell{\underline{27.90}}{\underline{27.90}}{(2m)}
& \wdacell{33.68}{\textendash}{(4m)}
& \wdacell[22.8pt]{\textbf{27.68}}{\textbf{27.68}}{(45m)}
\\[2pt]

\midrule
\multicolumn{7}{c}{ImageNet, \(r=\infty,\ \varepsilon=4/255\), clean accuracy \(=52.92\)}
\\
\midrule

\(1\)
& \wdacell{1.30}{\textendash}{(10m)}
& \wdacell{0.62}{\textendash}{(7m)}
& \wdacell{0.96}{\textendash}{(7m)}
& \wdacell{0.62}{\textendash}{(5m)}
& \wdacell{\textbf{0.00}}{\textendash}{(44m)}
& \wdacell[22.8pt]{\underline{0.34}}{\textendash}{(50m)}
\\[2pt]

\(2\)
& \wdacell{\textbf{2.65}}{\textendash}{(10m)}
& \wdacell{9.12}{\textendash}{(7m)}
& \wdacell{11.81}{\textendash}{(7m)}
& \wdacell{\underline{5.73}}{\textendash}{(5m)}
& \wdacell{8.29}{15.26}{(44m)}
& \wdacell[22.8pt]{7.42}{\textendash}{(50m)}
\\[2pt]

\(\infty\)
& \wdacell{26.80}{26.80}{(10m)}
& \wdacell{29.34}{29.36}{(7m)}
& \wdacell{27.34}{27.34}{(7m)}
& \wdacell{\underline{25.66}}{\underline{25.66}}{(5m)}
& \wdacell{30.54}{\textendash}{(44m)}
& \wdacell[22.8pt]{\textbf{25.26}}{\textbf{25.26}}{(50m)}
\\[2pt]

\bottomrule
\end{tabular*}
\end{table}

The paired entries in Table~\ref{tab:wda-main-results} show that WDA yields classification accuracy no higher than Direct whenever both results are available.
WDA using W-PGD for finding support at \(p=2\) reduces classification accuracy by \(6\)\textendash\(15\) percentage points relative to Direct W-PGD across the three models.
These gains do not require a substantially larger final support, since WDA uses at most \(N+1\) support points and Direct at most \(N\).
WDA searches a richer collection of candidate locations and optimizes the probability mass transported to them under the same Wasserstein constraint.
At \(p=\infty\), the available WDA and Direct pairs differ by at most \(0.06\) percentage points.
Among the attacks used for finding support, Algorithm~\ref{alg:error-point-search} yields the lowest WDA accuracy on all three models at \(p=2\), whereas AA yields the lowest at \(p=\infty\).
AA does not yield the lowest WDA accuracy at \(p=1\) or \(p=2\) in these comparisons, but its recorded runtime is  approximately \(5\)\textendash\(25\) times that of Algorithm~\ref{alg:error-point-search}.
To limit computational cost, our AA evaluation is restricted to the three models in Table~\ref{tab:wda-main-results}.

\begin{table}[tb]
\centering
\footnotesize
\caption{Comparison of attacks for finding support within WDA on CIFAR-10, CIFAR-100, and ImageNet.
Entries report mean classification accuracy (\%), with mean runtime in minutes in parentheses.
Full results for individual models appear in Tables~\ref{tab:wda-cifar10-model-results} and~\ref{tab:wda-cifar100-imagenet-model-results}.}
\label{tab:wda-attack-summary}

\setlength{\tabcolsep}{3pt}
\renewcommand{\arraystretch}{1.12}

\providecommand{\summaryruntime}[2]{}
\renewcommand{\summaryruntime}[2]{%
  \makebox[3.4em][r]{#1}%
  \hspace{0.5em}%
  \makebox[3.4em][l]{{\scriptsize\mdseries(#2)}}%
}

\begin{tabularx}{\linewidth}{@{}l@{\hspace{1em}}c*{5}{Y}@{}}
\toprule
Dataset
& \(p\)
& Alg.~\ref{alg:error-point-search}
& APGD-CE
& APGD-DLR
& MM
& W-PGD
\\
\midrule

\multirow{2}{*}{CIFAR-10}
& 1
& \summaryruntime{\textbf{17.97}}{21m}
& \summaryruntime{25.17}{56m}
& \summaryruntime{28.74}{57m}
& \summaryruntime{22.43}{25m}
& \summaryruntime{\underline{21.67}}{49m}
\\
& 2
& \summaryruntime{\textbf{29.52}}{21m}
& \summaryruntime{40.41}{56m}
& \summaryruntime{43.40}{57m}
& \summaryruntime{36.91}{25m}
& \summaryruntime{\underline{33.63}}{49m}
\\
\midrule

\multirow{2}{*}{CIFAR-100}
& 1
& \summaryruntime{\underline{2.80}}{22m}
& \summaryruntime{6.16}{45m}
& \summaryruntime{8.47}{45m}
& \summaryruntime{\textbf{2.43}}{19m}
& \summaryruntime{3.87}{44m}
\\
& 2
& \summaryruntime{\textbf{9.08}}{22m}
& \summaryruntime{20.63}{45m}
& \summaryruntime{23.11}{45m}
& \summaryruntime{\underline{13.53}}{19m}
& \summaryruntime{14.21}{44m}
\\
\midrule

\multirow{2}{*}{ImageNet}
& 1
& \summaryruntime{\textbf{8.84}}{48m}
& \summaryruntime{13.47}{35m}
& \summaryruntime{15.74}{36m}
& \summaryruntime{\underline{11.42}}{19m}
& \summaryruntime{13.53}{73m}
\\
& 2
& \summaryruntime{\textbf{17.71}}{48m}
& \summaryruntime{28.26}{35m}
& \summaryruntime{30.71}{36m}
& \summaryruntime{\underline{24.49}}{19m}
& \summaryruntime{25.06}{73m}
\\

\bottomrule
\end{tabularx}
\end{table}

We next compare our Algorithm~\ref{alg:error-point-search} against the remaining four comparison attacks within WDA on all 18 models.
Table~\ref{tab:wda-attack-summary} reports mean classification accuracy and runtime for each dataset, averaging equally over the models listed in Tables~\ref{tab:wda-cifar10-model-results} and~\ref{tab:wda-cifar100-imagenet-model-results}.
WDA using Algorithm~\ref{alg:error-point-search} obtains the lowest mean classification accuracy in five of the six settings in Table~\ref{tab:wda-attack-summary}.
At \(p=2\), its mean accuracy is \(4\)\textendash\(7\) percentage points lower than the best of the other four attacks within each dataset.
Its recorded mean runtime is lower than when using W-PGD on all three datasets and either APGD variant on CIFAR-10 and CIFAR-100.
The ablation in Figure~\ref{fig:wda-bisection-ablation} (Appendix~\ref{app:experimental-details}) shows that adding bisection refinements \eqref{eq:candidate-support-refinement} to the candidate support sets lowers mean WDA accuracy by \(3\)\textendash\(20\) percentage points at \(p=1,2\) across the four comparison attacks and three datasets.

\subsection{Surrogate loss misalignment: illustrations of Theorems~\ref{thm:losses-preserving-robust-error-order} and~\ref{thm:surrogate-loss-misalignment}}
\label{subsec:surrogate-misalignment-illustration}

Let \(\theta_\ast\) denote the \citet{sehwag2022robust} CIFAR-10 ResNet-18 model from Table~\ref{tab:wda-main-results}, and let \(\theta_i\), \(i\in\{1,\ldots,8\}\), denote the comparison models indexed in Figure~\ref{fig:surrogate-misalignment-1&2}.
For each model, we optimize classification error and CE loss separately over the same candidate support sets, using \eqref{eq:wda-risk} for \(p=1,2\) and \eqref{eq:wda-risk-infinity-closed-form} for \(p=\infty\).

Panel (a) of Figure~\ref{fig:surrogate-misalignment-1&2} uses the original CIFAR-10 test empirical distribution and support locations found by Algorithm~\ref{alg:error-point-search}.
For panel (b), we select an unmodified CIFAR-10 image with its original label, \(z_{\rm con}^{(i)}\), and set \(P_{\rm con}^{(i)}=\delta_{z_{\rm con}^{(i)}}\) as the shared reference distribution for each model pair.
For each model, we retain a complete support set from an Algorithm~\ref{alg:error-point-search} run with the same label and add \(z_{\rm con}^{(i)}\).
The images are selected so that one model misclassifies \(z_{\rm con}^{(i)}\), while the other correctly classifies every point in its augmented support set.
The CE loss at \(z_{\rm con}^{(i)}\) is maximal on each model's augmented support set, but is lower for the model that misclassifies it.
Consequently, both models' WDA problems admit the optimal choice \(Q_{\rm con}^{(i)}=P_{\rm con}^{(i)}=\delta_{z_{\rm con}^{(i)}}\) for both losses and every displayed \(p\).

The panel labels \(Q_{\rm adv}^{(i)}\) and \(Q_{\rm con}^{(i)}\) denote the corresponding WDA laws, with dependence on the model, loss, and \(p\) suppressed.
For \(j\in\{\ast,i\}\), write \(\mathrm{Acc}(j) =100\bigl(1-\widetilde{\mathcal R}_p(\theta_j;\varepsilon,\1)\bigr)\) and \(\mathrm{CE\,loss}(j) =\widetilde{\mathcal R}_p(\theta_j;\varepsilon,\ell_{\rm CE})\) with dependence on the reference distribution, candidate sets, and \(p\) suppressed.
Figure~\ref{fig:surrogate-misalignment-1&2} plots \(-(\mathrm{Acc}(\ast)-\mathrm{Acc}(i))\) and \(\mathrm{CE\,loss}(\ast)-\mathrm{CE\,loss}(i)\), where opposite signs indicate surrogate loss misalignment.
Figure~\ref{fig:constructed-source-images} displays the selected images \(z_{\rm con}^{(i)}\) with their original labels and reports both models' predictions and CE losses.

\begin{figure}[tb]
\centering
\includegraphics[width=0.97\linewidth]{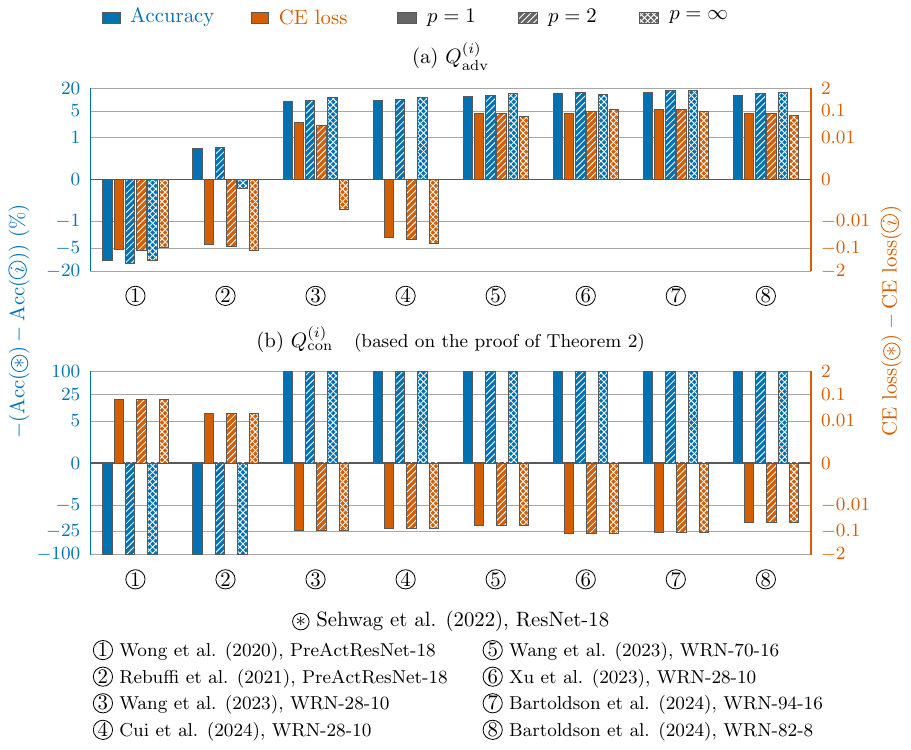}
\caption{Surrogate loss misalignment on CIFAR-10, illustrated using WDA risks for classification error and CE loss.}
\label{fig:surrogate-misalignment-1&2}
\end{figure}

\begin{figure}[tb]
\centering
\includegraphics[width=0.8\linewidth]{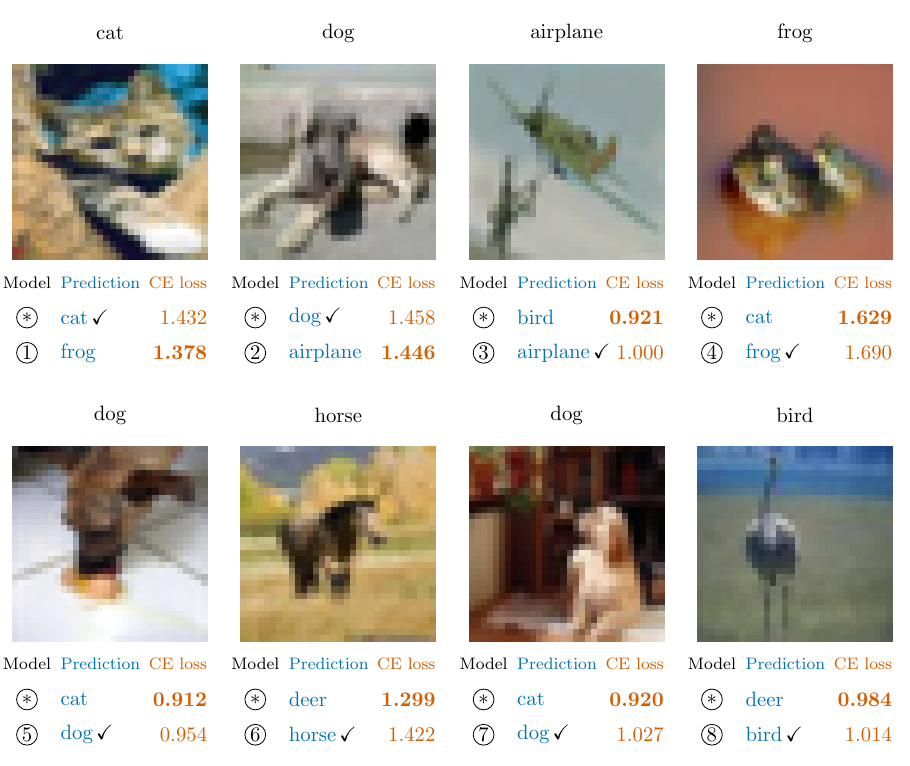}
\caption{Clean CIFAR-10 images used in the constructions of Figure~\ref{fig:surrogate-misalignment-1&2}(b).}
\label{fig:constructed-source-images}
\end{figure}

In panel (a) of Figure~\ref{fig:surrogate-misalignment-1&2}, model~{\small\textcircled{4}} has higher accuracy yet higher CE loss than model~{\small\textcircled{\(\ast\)}} for every displayed \(p\), illustrating the ordering disagreement addressed by Theorem~\ref{thm:losses-preserving-robust-error-order}.
Panel (b) draws on the atomic reference distributions used in the proof of Theorem~\ref{thm:surrogate-loss-misalignment}, here with fixed pretrained classifiers.
For each pair and every displayed \(p\), the classifier with smaller CE loss has classification error \(1\), while the other has classification error \(0\).
Within the label-preserving Wasserstein setting of Theorems~\ref{thm:losses-preserving-robust-error-order} and~\ref{thm:surrogate-loss-misalignment}, these comparisons show that maximizing classification error and CE loss over the same feasible candidate distributions for each pretrained neural classifier can give opposite orderings of the classifiers.

\section{Conclusion}
\label{sec:conclusion}

Surrogate loss misalignment can persist even when WDRO objectives are evaluated exactly.
Under label-preserving transportation, we show that WDRO risks give the same ordering of hypotheses as robust classification error uniformly over empirical reference distributions, radii, and measurable scoring functions exactly when the loss is a positive multiple of the classification error loss plus a label-dependent offset.
For a class of surrogate losses including CE and DLR, a hypothesis with robust classification error arbitrarily close to one can have a lower WDRO risk under the surrogate loss than a hypothesis with zero robust classification error.
Both results hold for every Wasserstein exponent \(p\in[1,\infty]\).

To target classification error directly, WDA separates finding support from optimizing transport weights under the Wasserstein constraint.
Maximizing classification error over these weights provides computable lower bounds on robust classification error, with tightness depending on the candidate support sets.
At \(p=2\), WDA using W-PGD for finding support attains lower classification accuracy than Direct W-PGD, and WDA using Algorithm~\ref{alg:error-point-search} yields the lowest classification accuracy on all 18 tested models across CIFAR-10, CIFAR-100, and ImageNet.
These results support using surrogate losses to guide candidate search while optimizing classification error in the weight generation stage.
Future work could develop distributional attacks beyond fixed finite candidate support sets while directly targeting classification error under the label-preserving Wasserstein constraint.

\appendix

\section{Proofs of Main Results}
\label{app:proofs}

\subsection{Proof of Theorem~\ref{thm:losses-preserving-robust-error-order}}
\label{app:proof-losses-preserving-robust-error-order}

\begin{proof}
We prove the equivalence by showing
\[
\hyperref[cond:robust-order-B]{(B)}
\Rightarrow
\hyperref[cond:robust-order-A0]{(A_0)},
\qquad
\hyperref[cond:robust-order-B]{(B)}
\Rightarrow
\hyperref[cond:robust-order-Aplus]{(A_+)},
\qquad
\hyperref[cond:robust-order-A0]{(A_0)}
\Rightarrow
\hyperref[cond:robust-order-B]{(B)},
\qquad
\hyperref[cond:robust-order-Aplus]{(A_+)}
\Rightarrow
\hyperref[cond:robust-order-B]{(B)}.
\]

\phantomsection \label{part:robust-order-1}
\paragraph{\texorpdfstring{Part 1: \(\hyperref[cond:robust-order-B]{(B)}\Rightarrow\hyperref[cond:robust-order-A0]{(A_0)}\).}{Part 1: (B) => (A0).}} 
Assume \(\hyperref[cond:robust-order-B]{(B)}\).
Fix \(\PN\) and a pair \(\theta,\theta'\).
By \(\hyperref[cond:robust-order-B]{(B)}\),
\[
\begin{aligned}
\mathbb E_{\PN}\!\left[\ell(Y,f_\theta(X))\right]
&=
\sum_{y=1}^K b_y\,\PN(Y=y)
+
a\,\mathbb E_{\PN}\!\left[\1(Y,f_\theta(X))\right],
\\
\mathbb E_{\PN}\!\left[\ell(Y,f_{\theta'}(X))\right]
&=
\sum_{y=1}^K b_y\,\PN(Y=y)
+
a\,\mathbb E_{\PN}\!\left[\1(Y,f_{\theta'}(X))\right].
\end{aligned}
\]
The same term \(\sum_{y=1}^K b_y\,\PN(Y=y)\) appears in both equations and \(a>0\).
Therefore
\[
\mathbb E_{\PN}\!\left[\1(Y,f_\theta(X))\right]
>
\mathbb E_{\PN}\!\left[\1(Y,f_{\theta'}(X))\right]
\Leftrightarrow
\mathbb E_{\PN}\!\left[\ell(Y,f_\theta(X))\right]
>
\mathbb E_{\PN}\!\left[\ell(Y,f_{\theta'}(X))\right].
\]
This proves \(\hyperref[cond:robust-order-B]{(B)}\Rightarrow\hyperref[cond:robust-order-A0]{(A_0)}\).

\phantomsection \label{part:robust-order-2}
\paragraph{\texorpdfstring{Part 2: \(\hyperref[cond:robust-order-B]{(B)}\Rightarrow\hyperref[cond:robust-order-Aplus]{(A_+)}\).}{Part 2: (B) => (A+).}} 
Assume \(\hyperref[cond:robust-order-B]{(B)}\).
Fix \(\PN\), \(\varepsilon>0\), and a pair \(\theta,\theta'\).
For any law \(Q\) satisfying \(W_p(\PN,Q)\le\varepsilon\), the label-preserving cost \(d_r\) in \eqref{eq:label-preserving-norm-cost} implies \(Q(Y=y)=\PN(Y=y)\) for every \(y\in[K]\).
Hence,
\[
\begin{aligned}
\mathbb E_Q\!\left[\ell(Y,f_\theta(X))\right]
&=
\sum_{y=1}^K b_y Q(Y=y)
+
a\,\mathbb E_Q\!\left[\1(Y,f_\theta(X))\right]
\\
&=
\sum_{y=1}^K b_y \PN(Y=y)
+
a\,\mathbb E_Q\!\left[\1(Y,f_\theta(X))\right].
\end{aligned}
\]
Since \(\sum_{y=1}^K b_y \PN(Y=y)\) is independent of \(Q\), and \(a>0\), taking the supremum over all feasible \(Q\) therefore yields
\[
\begin{aligned}
&\mathcal R_p(\theta;\varepsilon,\ell)
=
\sup_{Q:W_p(\PN,Q)\le\varepsilon}
\mathbb E_Q\!\left[\ell(Y,f_\theta(X))\right]
\\
&=
\sum_{y=1}^K b_y\,\PN(Y=y)
+
a
\sup_{Q:W_p(\PN,Q)\le\varepsilon}
\mathbb E_Q\!\left[\1(Y,f_\theta(X))\right]
=
\sum_{y=1}^K b_y\,\PN(Y=y)
+
a\,\mathcal R_p(\theta;\varepsilon,\1).
\end{aligned}
\]
Therefore, for a pair \(\theta, \theta'\), we have
\[
\mathcal R_p(\theta;\varepsilon,\ell)
=
\sum_{y=1}^K b_y\,\PN(Y=y)
+
a\,\mathcal R_p(\theta;\varepsilon,\1),
\]
\[
\mathcal R_p(\theta';\varepsilon,\ell)
=
\sum_{y=1}^K b_y\,\PN(Y=y)
+
a\,\mathcal R_p(\theta';\varepsilon,\1).
\]
The same term \(\sum_{y=1}^K b_y\,\PN(Y=y)\) appears in both equations and \(a>0\), so
\[
\mathcal R_p(\theta;\varepsilon,\1)
>
\mathcal R_p(\theta';\varepsilon,\1)
\Leftrightarrow
\mathcal R_p(\theta;\varepsilon,\ell)
>
\mathcal R_p(\theta';\varepsilon,\ell).
\]
This proves \(\hyperref[cond:robust-order-B]{(B)}\Rightarrow\hyperref[cond:robust-order-Aplus]{(A_+)}\).

\phantomsection \label{part:robust-order-3}
\paragraph{\texorpdfstring{Part 3: \(\hyperref[cond:robust-order-A0]{(A_0)}\Rightarrow\hyperref[cond:robust-order-B]{(B)}\).}{Part 3: (A0) => (B).}} 
Assume \(\hyperref[cond:robust-order-A0]{(A_0)}\).
Only the forward implication (\(\Rightarrow\)) in \(\hyperref[cond:robust-order-A0]{(A_0)}\) is used in the following argument.
We will show
\begin{enumerate}
\item\label{part:robust-order-3.1}
\(\ell(y,w)>\ell(y,v)\) for every label \(y\), every correct score \(v\) and every error score \(w\).
\item\label{part:robust-order-3.2}
\(\ell(y,v)=\ell(y,v')\) for every label \(y\) and every pair of correct scores \(v,v'\).
\item\label{part:robust-order-3.3}
\(\ell(y,w)=\ell(y,w')\) for every label \(y\) and every pair of error scores \(w,w'\).
\item\label{part:robust-order-3.4}
\(\ell(y,w)-\ell(y,v)\) is a positive constant independent of \(y\) for every label \(y\), every correct score \(v\) and every error score \(w\).
\end{enumerate}
Together these imply \(\ell(y,s)=b_y+a\,\1(y,s)\), where \(b_y=\ell(y,v)\) for any correct score \(v\) for label \(y\), and where \(a=\ell(y,w)-\ell(y,v)>0\) for any label \(y\), any correct score \(v\) and any error score \(w\).
Setting \(s=f_\theta(x)\) gives \(\hyperref[cond:robust-order-B]{(B)}\).

\emph{\hyperref[part:robust-order-3.1]{Part 3.1}: \(\ell(y,w)>\ell(y,v)\).}
Fix a label \(y\), a correct score \(v\), and an error score \(w\).
Fix any \(x\in\mathbb R^n\), and set \(\PN=\delta_{(x,y)}\).
Set constant scoring functions \(f_\theta,f_{\theta'}\) by
\[
f_\theta(\cdot)\equiv w,
\qquad
f_{\theta'}(\cdot)\equiv v.
\]
Then \(\mathbb E_{\PN}\!\left[\1(Y,f_{\theta}(X))\right]=1>0=\mathbb E_{\PN}\!\left[\1(Y,f_{\theta'}(X))\right]\).
By \(\hyperref[cond:robust-order-A0]{(A_0)}\),
\[
\mathbb E_{\PN}\!\left[\1(Y,f_{\theta}(X))\right]
>
\mathbb E_{\PN}\!\left[\1(Y,f_{\theta'}(X))\right]
\Rightarrow
\mathbb E_{\PN}\!\left[\ell(Y,f_{\theta}(X))\right]
>
\mathbb E_{\PN}\!\left[\ell(Y,f_{\theta'}(X))\right].
\]
Since \(\PN=\delta_{(x,y)}\), this is equivalent to \(\ell(y,w)>\ell(y,v)\).

\emph{\hyperref[part:robust-order-3.2]{Part 3.2}: \(\ell(y,v)=\ell(y,v')\).}
Suppose that \(\ell(y,\cdot)\) is not constant on correct scores.
Then there exist correct scores \(\underline v,\overline v\) for label \(y\) such that \(\ell(y,\underline v)<\ell(y,\overline v)\).
Fix any error score \(w\) for label \(y\).
By \hyperref[part:robust-order-3.1]{Part~3.1}, \(\ell(y,w)>\ell(y,\overline v)>\ell(y,\underline v)\).
Therefore \(0<\frac{\ell(y,\overline v)-\ell(y,\underline v)}{\ell(y,w)-\ell(y,\underline v)}<1\).
Choose \(\lambda\) such that
\[
0<\lambda<
\frac{\ell(y,\overline v)-\ell(y,\underline v)}
{\ell(y,w)-\ell(y,\underline v)}.
\]
Fix two distinct points \(x_1,x_2\in\mathbb R^n\), and set \(\PN=(1-\lambda)\delta_{(x_1,y)}+\lambda\delta_{(x_2,y)}\).
Set piecewise-constant scoring function \(f_{\theta}\) and constant scoring function \(f_{\theta'}\) by
\[
f_{\theta}(x)=
\begin{cases}
w, & x=x_2,\\
\underline v, & x\ne x_2,
\end{cases}
\qquad
f_{\theta'}(x)\equiv\overline v.
\]
Then \(\mathbb E_{\PN}\!\left[\1(Y,f_{\theta}(X))\right]=\lambda>0=\mathbb E_{\PN}\!\left[\1(Y,f_{\theta'}(X))\right]\).
However,
\[
\begin{aligned}
\mathbb E_{\PN}\!\left[\ell(Y,f_{\theta}(X))\right]
-
\mathbb E_{\PN}\!\left[\ell(Y,f_{\theta'}(X))\right]
&=
(1-\lambda)\ell(y,\underline v)
+\lambda\ell(y,w)
-
\ell(y,\overline v)
\\
&=
\ell(y,\underline v)-\ell(y,\overline v)
+\lambda\bigl(\ell(y,w)-\ell(y,\underline v)\bigr)
<0.
\end{aligned}
\]
This contradicts \(\hyperref[cond:robust-order-A0]{(A_0)}\).
Hence \(\ell(y,\cdot)\) is constant on correct scores for label \(y\).
We denote this common constant by \(b_y\).

\emph{\hyperref[part:robust-order-3.3]{Part 3.3}: \(\ell(y,w)=\ell(y,w')\).}
Suppose that \(\ell(y,\cdot)\) is not constant on error scores.
Then there exist error scores \(\underline w,\overline w\) for label \(y\) such that \(\ell(y,\underline w)<\ell(y,\overline w)\).
Fix any correct score \(v\) for label \(y\).
By \hyperref[part:robust-order-3.1]{Part~3.1} and \hyperref[part:robust-order-3.2]{Part~3.2}, \(\ell(y,\overline w)-b_y>\ell(y,\underline w)-b_y>0\).
Therefore \(\frac12<\frac{\ell(y,\overline w)-b_y}{\ell(y,\overline w)-b_y+\ell(y,\underline w)-b_y}<1\).
Choose \(\lambda\) such that
\[
\frac12
<
\lambda
<
\frac{\ell(y,\overline w)-b_y}
{\ell(y,\overline w)-b_y+\ell(y,\underline w)-b_y}.
\]
Fix two distinct points \(x_1,x_2\in\mathbb R^n\), and set \(\PN=\lambda\delta_{(x_1,y)}+(1-\lambda)\delta_{(x_2,y)}\).
Set piecewise-constant scoring functions \(f_{\theta}, f_{\theta'}\) by
\[
f_{\theta}(x)
=
\begin{cases}
\underline w, & x=x_1,\\
v, & x\ne x_1,
\end{cases}
\qquad
f_{\theta'}(x)
=
\begin{cases}
\overline w, & x=x_2,\\
v, & x\ne x_2.
\end{cases}
\]
Then \(\mathbb E_{\PN}\!\left[\1(Y,f_{\theta}(X))\right]=\lambda>1-\lambda=\mathbb E_{\PN}\!\left[\1(Y,f_{\theta'}(X))\right]\).
However,
\[
\begin{aligned}
\mathbb E_{\PN}\!\left[\ell(Y,f_{\theta}(X))\right]
-
\mathbb E_{\PN}\!\left[\ell(Y,f_{\theta'}(X))\right]
&=
\lambda\bigl(\ell(y,\underline w)-b_y\bigr)
-
(1-\lambda)\bigl(\ell(y,\overline w)-b_y\bigr)
\\
&=
\lambda
\bigl(
\ell(y,\underline w)-b_y+\ell(y,\overline w)-b_y
\bigr)
-
\bigl(\ell(y,\overline w)-b_y\bigr)
<0,
\end{aligned}
\]
where the last inequality follows from the choice of \(\lambda\).
This contradicts \(\hyperref[cond:robust-order-A0]{(A_0)}\).
Therefore \(\ell(y,\cdot)\) is constant on error scores for label \(y\).
By \hyperref[part:robust-order-3.2]{Part~3.2}, \(\ell(y,\cdot)\) is also constant on correct scores for label \(y\), and its common value on correct scores is \(b_y\).
Fix any correct score \(v\) and any error score \(w\) for label \(y\).
By \hyperref[part:robust-order-3.1]{Part~3.1}, \(\ell(y,w)>\ell(y,v)=b_y\).
Let \(a_y=\ell(y,w)-\ell(y,v)=\ell(y,w)-b_y>0\).
We write \(\ell(y,w)=b_y+a_y\) as the common value on error scores for some \(a_y>0\).
Therefore,
\[
\ell(y,s)
=
\begin{cases}
b_y + a_y, & s \text{ is an error score}\\
b_y, & s \text{ is a correct score}
\end{cases}
=
b_y+a_y\1(y,s).
\]

\emph{\hyperref[part:robust-order-3.4]{Part 3.4}: \(a_y=a_{y'}=:a>0\).}
It remains to show that the positive constant \(a_y\) is common across labels.
Suppose not.
Then there exist two labels \(y\ne y'\) such that \(a_y<a_{y'}\).
Since \(a_y<a_{y'}\), we have \(\frac12<\frac{a_{y'}}{a_y+a_{y'}}<1\).
Choose \(\lambda\) such that
\[
\frac12
<
\lambda
<
\frac{a_{y'}}{a_y+a_{y'}}.
\]
Then \(\lambda>1-\lambda\) and \(\lambda a_y-(1-\lambda)a_{y'}<0\).
Fix two distinct points \(x_1,x_2\in\mathbb R^n\), and set \(\PN=\lambda\delta_{(x_1,y)}+(1-\lambda)\delta_{(x_2,y')}\).
Fix any correct score \(v^{(y)}\) and error score \(w^{(y)}\) for label \(y\).
Fix any correct score \(v^{(y')}\) and error score \(w^{(y')}\) for label \(y'\).
Set piecewise-constant scoring functions \(f_{\theta},f_{\theta'}\) by
\[
f_{\theta}(x)
=
\begin{cases}
w^{(y)}, & x=x_1,\\
v^{(y')}, & x\ne x_1,
\end{cases}
\qquad
f_{\theta'}(x)
=
\begin{cases}
w^{(y')}, & x=x_2,\\
v^{(y)}, & x\ne x_2.
\end{cases}
\]
Then \(\mathbb E_{\PN}\!\left[\1(Y,f_{\theta}(X))\right]=\lambda>1-\lambda=\mathbb E_{\PN}\!\left[\1(Y,f_{\theta'}(X))\right]\).
However, since \(\ell(y,s)=b_y+a_y\1(y,s)\) and \(\ell(y',s)=b_{y'}+a_{y'}\1(y',s)\), we have
\[
\begin{aligned}
&\mathbb E_{\PN}\!\left[\ell(Y,f_{\theta}(X))\right]
-
\mathbb E_{\PN}\!\left[\ell(Y,f_{\theta'}(X))\right]
\\
&=
\left(\lambda (b_y+a_y) + (1-\lambda)b_{y'}\right)
-
\left(\lambda b_y + (1-\lambda)(b_{y'}+a_{y'})\right)
=
\lambda a_y-(1-\lambda)a_{y'}
<0.
\end{aligned}
\]
This contradicts \(\hyperref[cond:robust-order-A0]{(A_0)}\).
Hence \(a_y=a_{y'}\) for all \(y,y'\in[K]\).
Writing the common value as \(a>0\), we have, for every label \(y\in[K]\) and every score \(s\in\mathbb R^K\), \(\ell(y,s)=b_y+a\1(y,s)\).
This proves \(\hyperref[cond:robust-order-A0]{(A_0)}\Rightarrow\hyperref[cond:robust-order-B]{(B)}\).
Combining this implication with \hyperref[part:robust-order-1]{Part~1} gives \(\hyperref[cond:robust-order-A0]{(A_0)}\Leftrightarrow\hyperref[cond:robust-order-B]{(B)}\).

\phantomsection \label{part:robust-order-4}
\paragraph{\texorpdfstring{Part 4: \(\hyperref[cond:robust-order-Aplus]{(A_+)}\Rightarrow\hyperref[cond:robust-order-B]{(B)}\).}{Part 4: (A+) => (B).}}
By the argument in \hyperref[part:robust-order-3]{Part~3}, it suffices to derive the forward implication (\(\Rightarrow\)) in \(\hyperref[cond:robust-order-A0]{(A_0)}\) from \(\hyperref[cond:robust-order-Aplus]{(A_+)}\).
Assume \(\hyperref[cond:robust-order-Aplus]{(A_+)}\).
Its forward implication gives, for every \(\varepsilon>0\), every empirical law \(\PN\), and every pair of hypotheses \(\theta,\theta'\),
\[
\mathcal R_p(\theta;\varepsilon,\1)
>
\mathcal R_p(\theta';\varepsilon,\1)
\Rightarrow
\mathcal R_p(\theta;\varepsilon,\ell)
>
\mathcal R_p(\theta';\varepsilon,\ell).
\]
For any empirical law \(\widehat P_N=\sum_{i=1}^N\mu_i\delta_{(\widehat x_i,\widehat y_i)}\) and pair of scoring functions \(f_{\widehat\theta},f_{\widehat\theta'}\), denote
\begin{equation}
\begin{aligned}
\Delta_{\1}(\widehat P_N, \widehat \theta, \widehat \theta')
&:=
\mathbb E_{\widehat P_N}\!\left[\1(Y,f_{\widehat \theta}(X))\right]
-
\mathbb E_{\widehat P_N}\!\left[\1(Y,f_{\widehat \theta'}(X))\right]
\\
&=
\sum_{i=1}^N\mu_i\left(\1(\widehat y_i,f_{\widehat\theta}(\widehat x_i))
-
\1(\widehat y_i,f_{\widehat\theta'}(\widehat x_i))\right),
\end{aligned}
\label{eq:empirical-error-gap}
\end{equation}
\begin{equation}
\begin{aligned}
\Delta_\ell(\widehat P_N, \widehat \theta, \widehat \theta')
&:=
\mathbb E_{\widehat P_N}\!\left[\ell(Y,f_{\widehat \theta}(X))\right]
-
\mathbb E_{\widehat P_N}\!\left[\ell(Y,f_{\widehat \theta'}(X))\right]
\\
&=
\sum_{i=1}^N\mu_i\left(\ell(\widehat y_i,f_{\widehat\theta}(\widehat x_i))
-
\ell(\widehat y_i,f_{\widehat\theta'}(\widehat x_i))\right).
\end{aligned}
\label{eq:empirical-loss-gap}
\end{equation}
We will show that \(\Delta_{\1}(\widehat P_N,\widehat\theta,\widehat\theta')>0\) implies \(\Delta_\ell(\widehat P_N,\widehat\theta,\widehat\theta')>0\).

\phantomsection \label{part:robust-order-4.1}
\emph{Part 4.1: We will show that if \(\Delta_\ell(\widetilde P_N, \widehat \theta, \widehat \theta')\ge0\) then \(\Delta_\ell(\widehat P_N, \widehat \theta, \widehat \theta')>0\) where \(\widetilde P_N\) satisfies \(\widehat P_N=(1-\lambda\mu_j)\widetilde P_N+\lambda\mu_j\delta_{(\widehat x_j,\widehat y_j)}\).} 

Fix \(\widehat P_N\) with  \(\Delta_{\1}(\widehat P_N, \widehat \theta, \widehat \theta')>0\).
By \eqref{eq:empirical-error-gap}, there exists \(j\in[N]\) such that
\begin{equation}
\1\!\left(\widehat y_j,f_{\widehat\theta}(\widehat x_j)\right)=1,
\qquad
\1\!\left(\widehat y_j,f_{\widehat\theta'}(\widehat x_j)\right)=0.
\label{eq:error-gap}
\end{equation}
Fix any \(x\in\mathbb R^n\), set \(\PN=\delta_{(x,\widehat y_j)}\) and set constant scoring functions \(f_\theta, f_{\theta'}\) by
\[
f_\theta(\cdot)\equiv f_{\widehat\theta}(\widehat x_j),
\qquad
f_{\theta'}(\cdot)\equiv f_{\widehat\theta'}(\widehat x_j).
\]
For every \(\varepsilon>0\), because \(d_r\) preserves labels and the scoring functions are constant, we have
\[
\begin{aligned}
\mathcal R_p(\theta;\varepsilon,\1)
&=
\sup_{Q:W_p(\PN,Q)\le\varepsilon}
\mathbb E_Q\!\left[\1(Y,f_\theta(X))\right]
\\
&=
\sup_{Q:W_p(\PN,Q)\le\varepsilon}
\1(\widehat y_j,f_{\widehat\theta}(\widehat x_j))
=
\1(\widehat y_j,f_{\widehat\theta}(\widehat x_j))
=1.
\end{aligned}
\]
and similarly, \(\mathcal R_p(\theta';\varepsilon,\1)=\1(\widehat y_j,f_{\widehat\theta'}(\widehat x_j))=0\).
So \(\mathcal R_p(\theta;\varepsilon,\1)>\mathcal R_p(\theta';\varepsilon,\1)\). 
Therefore, by \(\hyperref[cond:robust-order-Aplus]{(A_+)}\), \(\mathcal R_p(\theta;\varepsilon,\ell)>\mathcal R_p(\theta';\varepsilon,\ell)\).
Likewise, label-preserving \(d_r\) and constant scoring functions give
\begin{equation}
\mathcal R_p(\theta;\varepsilon,\ell)-\mathcal R_p(\theta';\varepsilon,\ell)
=
\ell\!\left(\widehat y_j,f_{\widehat\theta}(\widehat x_j)\right)
-
\ell\!\left(\widehat y_j,f_{\widehat\theta'}(\widehat x_j)\right)
>
0.
\label{eq:loss-gap}
\end{equation}
Choose \(\lambda=\frac{1}{2}\min\{1,\Delta_{\1}(\widehat P_N, \widehat \theta, \widehat \theta')/\mu_j\}\).
Let \(\widetilde P_N\) be the empirical law obtained from \(\widehat P_N\) by removing a fraction \(\lambda\) of the mass \(\mu_j\) at \((\widehat x_j,\widehat y_j)\) and renormalizing
\[
\widetilde P_N
=
\frac{\widehat P_N-\lambda\mu_j\delta_{(\widehat x_j,\widehat y_j)}}{1-\lambda\mu_j}
=
\frac{
\sum_{i=1}^N\mu_i\delta_{(\widehat x_i,\widehat y_i)}
-
\lambda\mu_j\delta_{(\widehat x_j,\widehat y_j)}
}
{1-\lambda\mu_j}.
\]
Equivalently, \(\widehat P_N=(1-\lambda\mu_j)\widetilde P_N +\lambda\mu_j\delta_{(\widehat x_j,\widehat y_j)}\).
Plugging this into \eqref{eq:empirical-loss-gap}, we have
\[
\begin{aligned}
&\Delta_\ell(\widehat P_N, \widehat \theta, \widehat \theta')
=
\mathbb E_{\widehat P_N}\!\left[\ell(Y,f_{\widehat\theta}(X))\right]
-
\mathbb E_{\widehat P_N}\!\left[\ell(Y,f_{\widehat\theta'}(X))\right]
\\
&=
(1-\lambda\mu_j)
\left(
\mathbb E_{\widetilde P_N}\!\left[\ell(Y,f_{\widehat\theta}(X))\right]
-
\mathbb E_{\widetilde P_N}\!\left[\ell(Y,f_{\widehat\theta'}(X))\right]
\right)
+
\lambda\mu_j
\left[
\ell\!\left(\widehat y_j,f_{\widehat\theta}(\widehat x_j)\right)
-
\ell\!\left(\widehat y_j,f_{\widehat\theta'}(\widehat x_j)\right)
\right]
\\
&>
(1-\lambda\mu_j)\Delta_\ell(\widetilde P_N, \widehat \theta, \widehat \theta'),
\end{aligned}
\]
where the last inequality follows from \eqref{eq:loss-gap}. 
Therefore, in order to prove \(\Delta_\ell(\widehat P_N, \widehat \theta, \widehat \theta')>0\), it is enough to show that \(\Delta_\ell(\widetilde P_N, \widehat \theta, \widehat \theta')\ge0\), that is
\begin{equation}
\mathbb E_{\widetilde P_N}\!\left[\ell(Y,f_{\widehat\theta}(X))\right]
\ge
\mathbb E_{\widetilde P_N}\!\left[\ell(Y,f_{\widehat\theta'}(X))\right].
\label{eq:empirical-loss-gap-weak}
\end{equation}
By \eqref{eq:error-gap} and the choice of \(\lambda\), we also have
\begin{equation}
\begin{aligned}
&\Delta_\1(\widetilde P_N, \widehat \theta, \widehat \theta')
=
\mathbb E_{\widetilde P_N}\!\left[\1(Y,f_{\widehat\theta}(X))\right]
-
\mathbb E_{\widetilde P_N}\!\left[\1(Y,f_{\widehat\theta'}(X))\right]
\\
&=
\frac{
\left(\mathbb E_{\widehat P_N}\!\left[\1(Y,f_{\widehat \theta}(X))\right]
-
\mathbb E_{\widehat P_N}\!\left[\1(Y,f_{\widehat \theta'}(X))\right]\right)
-
\lambda\mu_j
\left(
\1(\widehat y_j,f_{\widehat\theta}(\widehat x_j))
-
\1(\widehat y_j,f_{\widehat\theta'}(\widehat x_j))
\right)
}
{1-\lambda\mu_j}
\\
&=
\frac{\Delta_{\1}(\widehat P_N, \widehat \theta, \widehat \theta')-\lambda\mu_j}{1-\lambda\mu_j}
\ge
\frac{\Delta_{\1}(\widehat P_N, \widehat \theta, \widehat \theta')-\frac{1}{2}\Delta_{\1}(\widehat P_N, \widehat \theta, \widehat \theta')}{1-\lambda\mu_j}
>
0.
\end{aligned}
\label{eq:empirical-error-gap-PN-tilde}
\end{equation}

\phantomsection \label{part:robust-order-4.2}
\emph{Part 4.2: We will prove that \(\Delta_\ell(\widetilde P_N, \widehat \theta, \widehat \theta')\ge0\) by constructing \(\PN\) and scoring functions \(f_\theta,f_{\theta'}\) such that \(\Delta_\ell(\PN,\theta,\theta')=\Delta_\ell(\widetilde P_N, \widehat \theta, \widehat \theta')\) and  \(\Delta_\ell(\PN,\theta,\theta')\ge0\).} 

Let \(\Lambda>0\) be fixed.
We now replace the locations \(\widehat x_i\) by new locations \(x_i\) that are far apart, while keeping the same labels, weights, and scores.
Fix pairwise disjoint open balls \(B_1,\ldots,B_N\subset\mathbb R^n\), each of radius \(\Lambda/2\) with respect to the norm \(\|\cdot\|_r\).
Let \(x_i\) denote the center of \(B_i\), so \(B_i=\left\{x\in\mathbb R^n:\|x-x_i\|_r<\Lambda/2\right\}\).
Figure~\ref{fig:disjoint-balls} gives a schematic illustration of this disjoint balls construction.
Set the law
\begin{figure}[htb]
\centering
\begin{tikzpicture}[>=Stealth, font=\small]
\def\r{1.3}
\coordinate (x1) at (0,0);
\coordinate (x2) at ({2*\r},{1.2*\r});
\coordinate (x3) at ({4.2*\r},0);
\draw[dashed] (x1) circle[radius=\r];
\draw[dashed] (x2) circle[radius=\r];
\draw[dashed] (x3) circle[radius=\r];
\fill (x1) circle (1.5pt);
\fill (x2) circle (1.5pt);
\fill (x3) circle (1.5pt);
\node[above] at ($(x1)+(0,\r)$) {\(B_1\)};
\node[above] at ($(x2)+(0,\r)$) {\(B_2\)};
\node[above] at ($(x3)+(0,\r)$) {\(B_3\)};
\node[below] at (x1) {\(x_1\)};
\node[below] at (x2) {\(x_2\)};
\node[below] at (x3) {\(x_3\)};
\draw[<->] (x1) -- node[below, midway] {\(\ge \Lambda\)} (x3);
\draw[->] (x2) -- node[above, midway] {\(\Lambda/2\)} +(\r,0);
\end{tikzpicture}
\caption{Schematic of the disjoint balls construction in \hyperref[part:robust-order-4.2]{Part~4.2}.}
\label{fig:disjoint-balls}
\end{figure}
\[
\PN
=
\frac{
\sum_{i=1}^N\mu_i\delta_{(x_i,\widehat y_i)}
-
\lambda\mu_j\delta_{(x_j,\widehat y_j)}
}
{1-\lambda\mu_j}.
\]
Set piecewise-constant scoring functions \(f_\theta,f_{\theta'}\) by
\[
f_\theta(x)
=
\begin{cases}
f_{\widehat\theta}(\widehat x_i), & x\in B_i,\ i\in[N],\\
f_{\widehat\theta}(\widehat x_1), & x\in\mathbb R^n\setminus\bigcup_{i=1}^N B_i,
\end{cases}
\qquad
f_{\theta'}(x)
=
\begin{cases}
f_{\widehat\theta'}(\widehat x_i), & x\in B_i,\ i\in[N],\\
f_{\widehat\theta'}(\widehat x_1), & x\in\mathbb R^n\setminus\bigcup_{i=1}^N B_i.
\end{cases}
\]
Then \(\PN\) has the same labels and weights as \(\widetilde P_N\), but at the separated locations \(x_i\), and the scoring functions \(f_\theta,f_{\theta'}\) reproduce the scores of \(f_{\widehat\theta},f_{\widehat\theta'}\) on the corresponding balls.
Consequently,
\begin{equation}
\begin{aligned}
\Delta_\1(\PN,\theta,\theta')
&=
\mathbb E_{\PN}\!\left[\1(Y,f_\theta(X))\right]
-
\mathbb E_{\PN}\!\left[\1(Y,f_{\theta'}(X))\right]
\\
&=
\mathbb E_{\widetilde P_N}\!\left[\1(Y,f_{\widehat\theta}(X))\right]
-
\mathbb E_{\widetilde P_N}\!\left[\1(Y,f_{\widehat\theta'}(X))\right]
= 
\Delta_\1(\widetilde P_N, \widehat \theta, \widehat \theta').
\end{aligned}
\label{eq:empirical-error-gap-PN}
\end{equation}
\begin{equation}
\begin{aligned}
\Delta_\ell(\PN,\theta,\theta')
&=
\mathbb E_{\PN}\!\left[\ell(Y,f_\theta(X))\right]
-
\mathbb E_{\PN}\!\left[\ell(Y,f_{\theta'}(X))\right]
\\
&=
\mathbb E_{\widetilde P_N}\!\left[\ell(Y,f_{\widehat\theta}(X))\right]
-
\mathbb E_{\widetilde P_N}\!\left[\ell(Y,f_{\widehat\theta'}(X))\right]
=
\Delta_\ell(\widetilde P_N, \widehat \theta, \widehat \theta').
\end{aligned}
\label{eq:empirical-loss-gap-PN}
\end{equation}
By \eqref{eq:empirical-error-gap-PN-tilde}, we also have \(\Delta_\1(\PN,\theta,\theta')=\frac{\Delta_{\1}(\widehat P_N, \widehat \theta, \widehat \theta')-\lambda\mu_j}{1-\lambda\mu_j}>0\).
We fix \(\varepsilon>0\) and show that we can construct \(\PN\) and scoring functions \(f_\theta,f_{\theta'}\) that satisfy \(\Delta_\ell(\PN,\theta,\theta')\ge0\).

\emph{Case 1: \(p=\infty\).}
Choose \(\Lambda = 3\varepsilon\).
Let \(Q\) satisfy \(W_\infty(\PN,Q)\le\varepsilon\).
Since \(d_r\) preserves labels and \(\varepsilon<\Lambda/2\), any feasible coupling moves mass starting from \((x_i,\widehat y_i)\) only to points \((x,\widehat y_i)\) with \(x\in B_i\).
Therefore, every feasible \(Q\) sees the same label and scores as the empirical law \(\PN\), and hence
\[
\mathbb E_Q\!\left[\1(Y,f_\theta(X))\right]
=
\mathbb E_{\PN}\!\left[\1(Y,f_\theta(X))\right],
\qquad
\mathbb E_Q\!\left[\1(Y,f_{\theta'}(X))\right]
=
\mathbb E_{\PN}\!\left[\1(Y,f_{\theta'}(X))\right],
\]
\[
\mathbb E_Q\!\left[\ell(Y,f_\theta(X))\right]
=
\mathbb E_{\PN}\!\left[\ell(Y,f_\theta(X))\right],
\qquad
\mathbb E_Q\!\left[\ell(Y,f_{\theta'}(X))\right]
=
\mathbb E_{\PN}\!\left[\ell(Y,f_{\theta'}(X))\right].
\]
Taking suprema over all \(Q\) with \(W_\infty(\PN,Q)\le\varepsilon\), we get
\[
\mathcal R_\infty(\theta;\varepsilon,\1)
-
\mathcal R_\infty(\theta';\varepsilon,\1)
=
\mathbb E_{\PN}\!\left[\1(Y,f_\theta(X))\right]
-
\mathbb E_{\PN}\!\left[\1(Y,f_{\theta'}(X))\right]
>
0,
\]
where the last inequality follows from \eqref{eq:empirical-error-gap-PN}.
Similarly,
\[
\mathcal R_\infty(\theta;\varepsilon,\ell)
-
\mathcal R_\infty(\theta';\varepsilon,\ell)
=
\mathbb E_{\PN}\!\left[\ell(Y,f_\theta(X))\right]
-
\mathbb E_{\PN}\!\left[\ell(Y,f_{\theta'}(X))\right].
\]
Since \(\mathcal R_\infty(\theta;\varepsilon,\1)>\mathcal R_\infty(\theta';\varepsilon,\1)\), applying \(\hyperref[cond:robust-order-Aplus]{(A_+)}\), we have \(\mathcal R_\infty(\theta;\varepsilon,\ell)>\mathcal R_\infty(\theta';\varepsilon,\ell)\). 
So \(\mathbb E_{\PN}\!\left[\ell(Y,f_\theta(X))\right]>\mathbb E_{\PN}\!\left[\ell(Y,f_{\theta'}(X))\right]\).
By \eqref{eq:empirical-loss-gap-PN}, this is equivalent to \(\mathbb E_{\widetilde P_N}\!\left[\ell(Y,f_{\widehat\theta}(X))\right]>\mathbb E_{\widetilde P_N}\!\left[\ell(Y,f_{\widehat\theta'}(X))\right]\), which implies \eqref{eq:empirical-loss-gap-weak}.

\emph{Case 2: \(1\le p<\infty\).}
Fix an arbitrary \(\gamma>0\) and let
\[
M_\ell
=
\max_{1\le i,k\le N}
\left|
\ell\!\left(\widehat y_i,f_{\widehat\theta}(\widehat x_k)\right)
-
\ell\!\left(\widehat y_i,f_{\widehat\theta}(\widehat x_i)\right)
\right|.
\]
Since \(\Delta_{\1}(\widehat P_N, \widehat \theta, \widehat \theta')-\lambda\mu_j>0\) by the choice of \(\lambda\), choose \(\Lambda\) so that
\[
\Lambda
=
3\varepsilon
\max\left\{
\left(\frac{1-\lambda\mu_j}{\Delta_{\1}(\widehat P_N, \widehat \theta, \widehat \theta')-\lambda\mu_j}\right)^{1/p},
\left(\frac{M_\ell}{\gamma}\right)^{1/p}
\right\}.
\]
Then
\[
\left(\frac{2\varepsilon}{\Lambda}\right)^p
<
\frac{\Delta_{\1}(\widehat P_N, \widehat \theta, \widehat \theta')-\lambda\mu_j}{1-\lambda\mu_j},
\qquad
M_\ell
\left(\frac{2\varepsilon}{\Lambda}\right)^p
<
\gamma.
\]
Let \(Q\) satisfy \(W_p(\PN,Q)\le\varepsilon\).
Choose an optimal coupling \(\pi\in\Pi(\PN,Q)\) such that \(\int d_r(z,z')^p\,d\pi(z,z')\le\varepsilon^p\).
Since \(d_r\) preserves labels, \(\pi\) transports mass only between points with the same label.
Let \(\alpha\) be the total \(\pi\)-mass transported from some source point \((x_i,\widehat y_i)\) to a point \((x,\widehat y_i)\) with \(x\notin B_i\):
\[
\alpha
=
\sum_{i=1}^N
\pi
\left(
\{(x_i,\widehat y_i)\}
\times
\{(x,\widehat y_i):x\notin B_i\}
\right).
\]
Every such move has \(d_r\)-cost at least \(\Lambda/2\).
Therefore \(\varepsilon^p\ge\int d_r(z,z')^p\,d\pi(z,z')\ge\left(\frac{\Lambda}{2}\right)^p\alpha\), and hence
\begin{equation}
\label{eq:escaping-mass-bound}
\alpha
\le
\left(\frac{2}{\Lambda}\right)^p(\varepsilon^p).
\end{equation}
For \(\1(Y,f_{\theta'}(X))\) under \(Q\), mass transported from \((x_i,\widehat y_i)\) to \((x,\widehat y_i)\) where \(x\in B_i\) sees the same scores as under \(\PN\), while mass transported to \((x,\widehat y_i)\) where \(x\notin B_i\) contributes at most one to the error indicator. 
Thus \eqref{eq:escaping-mass-bound} gives
\[
\mathbb E_Q\!\left[\1(Y,f_{\theta'}(X))\right]
\le
\mathbb E_{\PN}\!\left[\1(Y,f_{\theta'}(X))\right]
+
\left(\frac{2}{\Lambda}\right)^p(\varepsilon^p).
\]
Taking the supremum over all \(Q\) with \(W_p(\PN,Q)\le\varepsilon\), we obtain
\[
\mathcal R_p(\theta';\varepsilon,\1)
\le
\mathbb E_{\PN}\!\left[\1(Y,f_{\theta'}(X))\right]
+
\left(\frac{2\varepsilon}{\Lambda}\right)^p
<
\mathbb E_{\PN}\!\left[\1(Y,f_{\theta'}(X))\right]
+
\frac{\Delta_{\1}(\widehat P_N, \widehat \theta, \widehat \theta')-\lambda\mu_j}{1-\lambda\mu_j},
\]
where the last inequality follows from the choice of \(\Lambda\).
On the other hand, since \(Q=\PN\) is feasible, \(\mathcal R_p(\theta;\varepsilon,\1)\ge\mathbb E_{\PN}\!\left[\1(Y,f_\theta(X))\right]\).
Combining these two bounds with \eqref{eq:empirical-error-gap-PN} gives
\[
\mathcal R_p(\theta;\varepsilon,\1)
\ge
\mathbb E_{\PN}\!\left[\1(Y,f_{\theta'}(X))\right]
+
\frac{\Delta_{\1}(\widehat P_N, \widehat \theta, \widehat \theta')-\lambda\mu_j}{1-\lambda\mu_j}
>
\mathcal R_p(\theta';\varepsilon,\1).
\]
Applying \(\hyperref[cond:robust-order-Aplus]{(A_+)}\), we have \(\mathcal R_p(\theta;\varepsilon,\ell)>\mathcal R_p(\theta';\varepsilon,\ell)\).
For \(\ell(Y,f_{\theta}(X))\) under \(Q\), mass transported from \((x_i,\widehat y_i)\) to \((x,\widehat y_i)\) where \(x\in B_i\) sees the same scores as under \(\PN\), while mass transported to \((x,\widehat y_i)\) where \(x\notin B_i\) can increase the loss by at most \(M_\ell\).
Thus \eqref{eq:escaping-mass-bound} gives
\[
\mathbb E_Q\!\left[\ell(Y,f_\theta(X))\right]
\le
\mathbb E_{\PN}\!\left[\ell(Y,f_\theta(X))\right]
+
M_\ell
\left(\frac{2}{\Lambda}\right)^p(\varepsilon^p).
\]
Taking the supremum over all \(Q\) with \(W_p(\PN,Q)\le\varepsilon\), we obtain
\[
\mathcal R_p(\theta;\varepsilon,\ell)
\le
\mathbb E_{\PN}\!\left[\ell(Y,f_\theta(X))\right]
+
M_\ell
\left(\frac{2\varepsilon}{\Lambda}\right)^p
<
\mathbb E_{\PN}\!\left[\ell(Y,f_\theta(X))\right]
+
\gamma,
\]
where the last inequality follows from the choice of \(\Lambda\).
On the other hand, since \(Q=\PN\) is feasible, \(\mathcal R_p(\theta';\varepsilon,\ell)\ge\mathbb E_{\PN}\!\left[\ell(Y,f_{\theta'}(X))\right]\).
Combining these two bounds with \(\mathcal R_p(\theta;\varepsilon,\ell)>\mathcal R_p(\theta';\varepsilon,\ell)\) gives \(\mathbb E_{\PN}\!\left[\ell(Y,f_\theta(X))\right]-\mathbb E_{\PN}\!\left[\ell(Y,f_{\theta'}(X))\right]>-\gamma\).
By \eqref{eq:empirical-loss-gap-PN}, this is equivalent to \(\mathbb E_{\widetilde P_N}\!\left[\ell(Y,f_{\widehat\theta}(X))\right]-\mathbb E_{\widetilde P_N}\!\left[\ell(Y,f_{\widehat\theta'}(X))\right]>-\gamma\).
Since \(\gamma>0\) was arbitrary, we obtain \(\mathbb E_{\widetilde P_N}\!\left[\ell(Y,f_{\widehat\theta}(X))\right]-\mathbb E_{\widetilde P_N}\!\left[\ell(Y,f_{\widehat\theta'}(X))\right]\ge 0\), which proves \eqref{eq:empirical-loss-gap-weak}.

Thus \(\Delta_\ell(\widetilde P_N,\widehat\theta,\widehat\theta')\ge0\) holds for every \(p\in[1,\infty]\).
By \hyperref[part:robust-order-4.1]{Part~4.1}, we obtain \(\Delta_\ell(\widehat P_N,\widehat\theta,\widehat\theta')>0\) whenever \(\Delta_{\1}(\widehat P_N,\widehat\theta,\widehat\theta')>0\).
Applying the argument in \hyperref[part:robust-order-3]{Part~3} to this implication yields \(\hyperref[cond:robust-order-B]{(B)}\).
This proves \(\hyperref[cond:robust-order-Aplus]{(A_+)}\Rightarrow\hyperref[cond:robust-order-B]{(B)}\).
\end{proof}

\begin{remark}[Bounded input domains]
Theorem~\ref{thm:losses-preserving-robust-error-order} remains valid if the input space is a bounded ball \(\X_R=\{x\in\mathbb R^n:\|x\|_r\le R\}\), where \(R>0\).
Only \hyperref[part:robust-order-4.2]{Part~4.2} of the proof changes.
After fixing \(\widehat P_N\), \(j\), \(\lambda\), and \(\widetilde P_N\), choose \(x_1,\ldots,x_N\in\X_R\) with pairwise separation at least \(\Lambda_R>0\).
For example, when \(N\ge2\), placing \(x_i\)'s equally spaced with
\[
x_i
=
\left(-R+\frac{2R(i-1)}{N-1}\right)e_1,
\]
gives \(\Lambda_R=2R/(N-1)\).
Let \(B_i=\{x\in\X_R:\|x-x_i\|_r<\Lambda_R/2\}\).
The constructed law \(\PN\) and the scoring functions \(f_\theta,f_{\theta'}\) are defined exactly as in \hyperref[part:robust-order-4.2]{Part~4.2}, with \(\Lambda\) replaced by \(\Lambda_R\).

If \(p=\infty\), choose \(\varepsilon=\Lambda_R/3\).
Then the \(p=\infty\) argument is unchanged.

If \(1\le p<\infty\), fix \(\gamma>0\) and choose
\[
\varepsilon
=
\frac{\Lambda_R}{3}
\min\left\{
\left(\frac{\Delta_{\1}(\widehat P_N,\widehat\theta,\widehat\theta')-\lambda\mu_j}{1-\lambda\mu_j}\right)^{1/p},
\left(\frac{\gamma}{1+M_\ell}\right)^{1/p}
\right\}.
\]
The finite-\(p\) argument then gives \(\Delta_\ell(\widetilde P_N,\widehat\theta,\widehat\theta')>-\gamma\).
Since \(\hyperref[cond:robust-order-Aplus]{(A_+)}\) holds for every positive radius, it may be applied at the chosen \(\varepsilon\).
Letting \(\gamma\downarrow0\) gives \(\Delta_\ell(\widetilde P_N,\widehat\theta,\widehat\theta')\ge0\), and \hyperref[part:robust-order-4.1]{Part~4.1} completes the proof.
\end{remark}

\subsection{Proof of Theorem~\ref{thm:surrogate-loss-misalignment}}
\label{app:proof-surrogate-loss-misalignment}

\begin{proof}
\emph{Setup.}
Let \(x,y\) satisfy \hyperref[cond:loss-misalignment]{\((M)\)}, so
\[
\inf_{\theta:\,(x,y)\in\mathcal S_\theta}\ell(y,f_\theta(x))
<
\inf_{\theta:\,(x,y)\in\Z\setminus\mathcal S_\theta}\ell(y,f_\theta(x))
\le
\sup_{\theta:\,(x,y)\in\mathcal S_\theta}\ell(y,f_\theta(x)).
\]
By the strict left inequality, choose a hypothesis \(\underline{\theta}\) such that \((x,y)\in\mathcal S_{\underline{\theta}}\) and
\[
\ell(y,f_{\underline{\theta}}(x))
<
\inf_{\theta:\,(x,y)\in\Z\setminus\mathcal S_\theta}\ell(y,f_\theta(x)).
\]
Fix \(\varepsilon>0\) and \(\zeta\in(0,1)\).
Set \(\lambda=1-\zeta/2\).
Then \(\lambda\in(0,1)\) and \(\lambda>1-\zeta\).
Thus,
\[
(1-\lambda)\ell(y,f_{\underline{\theta}}(x))
+
\lambda
\inf_{\theta:\,(x,y)\in\Z\setminus\mathcal S_\theta}\ell(y,f_\theta(x))
<
\inf_{\theta:\,(x,y)\in\Z\setminus\mathcal S_\theta}\ell(y,f_\theta(x)).
\]
By the definition of the infimum, choose a hypothesis \(\theta'\) such that \((x,y)\in\Z\setminus\mathcal S_{\theta'}\) and
\[
(1-\lambda)\ell(y,f_{\underline{\theta}}(x))
+
\lambda\ell(y,f_{\theta'}(x))
<
\inf_{\theta:\,(x,y)\in\Z\setminus\mathcal S_\theta}\ell(y,f_\theta(x)).
\]
Since \(\theta' \in \{\theta: (x,y)\in\Z\setminus\mathcal S_{\theta}\}\), we also have
\[
\ell(y,f_{\underline{\theta}}(x))
<
\inf_{\theta:\,(x,y)\in\Z\setminus\mathcal S_\theta}\ell(y,f_\theta(x))
\le
\ell(y,f_{\theta'}(x)).
\]
By the right inequality in \hyperref[cond:loss-misalignment]{\((M)\)},
\[
(1-\lambda)\ell(y,f_{\underline{\theta}}(x))
+
\lambda\ell(y,f_{\theta'}(x))
<
\sup_{\theta:\,(x,y)\in\mathcal S_\theta}\ell(y,f_\theta(x)).
\]
By the definition of the supremum, choose a hypothesis \(\overline{\theta}\) such that \((x,y)\in\mathcal S_{\overline{\theta}}\) and
\[
(1-\lambda)\ell(y,f_{\underline{\theta}}(x))
+
\lambda\ell(y,f_{\theta'}(x))
<
\ell(y,f_{\overline{\theta}}(x)).
\]
We have therefore chosen hypotheses \(\underline{\theta}\), \(\theta'\), and \(\overline{\theta}\) such that \((x,y)\in\mathcal S_{\underline{\theta}}\), \(
(x,y)\in\Z\setminus\mathcal S_{\theta'}\), \((x,y)\in\mathcal S_{\overline{\theta}}\), and
\begin{equation}
\ell(y,f_{\underline{\theta}}(x))
<
\ell(y,f_{\theta'}(x)),
\label{eq:t2-low-error-loss-order}
\end{equation}
\begin{equation}
(1-\lambda)\ell(y,f_{\underline{\theta}}(x))
+
\lambda\ell(y,f_{\theta'}(x))
<
\ell(y,f_{\overline{\theta}}(x)).
\label{eq:t2-mixture-below-high-correct-loss}
\end{equation}
The two cases below use these inequalities together with the finite-\(p\) and \(p=\infty\) transport geometries illustrated in Figure~\ref{fig:t2-proof-transport} to construct \(\theta_1,\theta_2\) and \(\PN\) such that
\[
\mathcal R_p(\theta_1;\varepsilon,\ell)<\mathcal R_p(\theta_2;\varepsilon,\ell),
\qquad
\mathcal R_p(\theta_1;\varepsilon,\1)
>
1-\zeta
>
\mathcal R_p(\theta_2;\varepsilon,\1)
=
0.
\]

\begin{figure}[htb]
\centering
\begin{tikzpicture}[font=\small, every node/.style={align=center}]
\definecolor{Cerulean}{RGB}{0,114,178}
\definecolor{Vermilion}{RGB}{213,94,0}
\tikzset{
point/.style={
circle,
inner sep=0pt,
line width=0.45pt
},
pnbase/.style={
point,
draw=Cerulean
},
plbase/.style={
point,
draw=Vermilion
},
pnunit/.style={
pnbase,
fill=Cerulean,
minimum size=12.0pt
},
pnlambda/.style={
pnbase,
fill=Cerulean,
minimum size=4.8pt
},
pnremainder/.style={
pnbase,
fill=Cerulean,
minimum size=7.2pt
},
pllambda/.style={
plbase,
fill=Vermilion,
minimum size=4.8pt
},
plremainder/.style={
plbase,
fill=Vermilion,
minimum size=7.2pt
},
support/.style={
circle,
fill=black,
inner sep=0.75pt
},
axis/.style={
draw=black!55,
dashed,
line width=0.4pt
},
transport/.style={
-{Stealth[length=3.8pt,width=4.8pt]},
draw=black!80,
line width=0.55pt,
shorten <=5pt,
shorten >=5pt
},
dimension/.style={
{Stealth[length=3.5pt,width=4.5pt]}-{Stealth[length=3.5pt,width=4.5pt]},
draw=black!75,
line width=0.35pt
}
}

\def\panelwidth{4.7}
\def\ysource{1.50}
\def\ytarget{0.15}
\def\ydistance{-0.72}
\pgfmathsetmacro{\splitpoint}{\panelwidth/3}

\begin{scope}
\node at (2.35,2.60) {\textbf{(a)} \(1\le p<\infty\)};

\coordinate (a0s) at (0,\ysource);
\coordinate (a1s) at (\panelwidth,\ysource);
\coordinate (a0t) at (0,\ytarget);
\coordinate (a1t) at (\panelwidth,\ytarget);
\coordinate (a0d) at (0,\ydistance);
\coordinate (a1d) at (\panelwidth,\ydistance);

\draw[axis] (a0s) -- (a1s);
\draw[axis] (a0t) -- (a1t);

\node[anchor=east, text=Cerulean] at (-0.42,\ysource) {\(\PN\)};
\node[anchor=east, text=Vermilion] at (-0.42,\ytarget) {\(P_\lambda\)};

\node[pnunit] at (a0s) {};
\node[support] at (a1s) {};

\node[plremainder] at (a0t) {};
\node[pllambda] at (a1t) {};

\node[above=5pt, text=Cerulean] at (a0s) {\(1\)};
\node[above=5pt, text=Vermilion] at (a0t) {\(1-\lambda\)};
\node[above=5pt, , text=Vermilion] at (a1t) {\(\lambda\)};

\node[below=5pt] at (a0s) {\(z_0\)};
\node[below=5pt] at (a1s) {\(z_1\)};
\node[below=5pt] at (a0t) {\(z_0\)};
\node[below=5pt] at (a1t) {\(z_1\)};

\draw[transport]
(a0s)
.. controls +(1.0,0) and +(1.0,0) ..
(a0t);

\draw[transport]
(a0s)
.. controls +(1.0,0) and +(-0.65,0.7) ..
(a1t);

\draw[dimension] (a0d) -- (a1d);
\node[fill=white, inner xsep=2pt, inner ysep=0pt] at ($(a0d)!0.5!(a1d)$) {\(\varepsilon\lambda^{-1/p}\)};
\end{scope}

\begin{scope}[xshift=6.9cm]
\node at (2.35,2.60) {\textbf{(b)} \(p=\infty\)};

\coordinate (b0s) at (0,\ysource);
\coordinate (b1s) at (\splitpoint,\ysource);
\coordinate (b2s) at (\panelwidth,\ysource);
\coordinate (b0t) at (0,\ytarget);
\coordinate (b1t) at (\splitpoint,\ytarget);
\coordinate (b2t) at (\panelwidth,\ytarget);
\coordinate (b0d) at (0,\ydistance);
\coordinate (b1d) at (\splitpoint,\ydistance);
\coordinate (b2d) at (\panelwidth,\ydistance);

\draw[axis] (b0s) -- (b2s);
\draw[axis] (b0t) -- (b2t);

\node[anchor=east, text=Cerulean] at (-0.42,\ysource) {\(\PN\)};
\node[anchor=east, text=Vermilion] at (-0.42,\ytarget) {\(P_\lambda\)};

\node[pnlambda] at (b0s) {};
\node[support] at (b1s) {};
\node[pnremainder] at (b2s) {};

\node[support] at (b0t) {};
\node[pllambda] at (b1t) {};
\node[plremainder] at (b2t) {};

\node[above=5pt, text=Cerulean] at (b0s) {\(\lambda\)};
\node[above=5pt, text=Cerulean] at (b2s) {\(1-\lambda\)};
\node[above=5pt, text=Vermilion] at (b1t) {\(\lambda\)};
\node[above=5pt, text=Vermilion] at (b2t) {\(1-\lambda\)};

\node[below=5pt] at (b0s) {\(z_0\)};
\node[below=5pt] at (b1s) {\(z_1\)};
\node[below=5pt] at (b2s) {\(z_2\)};
\node[below=5pt] at (b0t) {\(z_0\)};
\node[below=5pt] at (b1t) {\(z_1\)};
\node[below=5pt] at (b2t) {\(z_2\)};

\draw[transport]
(b0s)
.. controls +(0.5,-1) and +(-0.4,0.4) ..
(b1t);

\draw[transport]
(b2s)
.. controls +(-1,-0.12) and +(-1,0.12) ..
(b2t);

\draw[dimension] (b0d) -- (b1d);
\node[fill=white, inner xsep=2pt, inner ysep=0pt] at ($(b0d)!0.5!(b1d)$) {\(\varepsilon\)};

\draw[dimension] (b1d) -- (b2d);
\node[fill=white, inner xsep=2pt, inner ysep=0pt] at ($(b1d)!0.5!(b2d)$) {\(2\varepsilon\)};
\end{scope}

\node at (5.8,-2)
{\(W_p(\textcolor{Cerulean}{\PN},Q)\le\varepsilon \Rightarrow Q(\{z_1\})\le\lambda,\qquad \textcolor{Vermilion}{P_\lambda}(\{z_1\})=\lambda\)};
\end{tikzpicture}
\caption{Transport geometries used in the proof of Theorem~\ref{thm:surrogate-loss-misalignment}.}
\label{fig:t2-proof-transport}
\end{figure}

\emph{Case 1: \(1\le p<\infty\).}
Set \(x_0=0\), \(x_1=\varepsilon\lambda^{-1/p}e_1\), \(z_0=(x_0,y)\), and \(z_1=(x_1,y)\).
Let \(\PN=\delta_{z_0}\).
Panel~(a) of Figure~\ref{fig:t2-proof-transport} illustrates this finite-\(p\) geometry.
Define a piecewise-constant scoring function \(f_{\theta_1}\) and a constant scoring function \(f_{\theta_2}\) by
\[
f_{\theta_1}(\tilde x)=
\begin{cases}
f_{\theta'}(x), & \tilde x=x_1,\\
f_{\underline{\theta}}(x), & \tilde x\ne x_1,
\end{cases}
\qquad
f_{\theta_2}(\tilde x)\equiv f_{\overline{\theta}}(x).
\]
Let \(Q\) satisfy \(W_p(\PN,Q)\le\varepsilon\), and set \(\nu=Q(\{z_1\})\).
Since \(W_p(\PN,Q)<\infty\) and \(\PN=\delta_{z_0}\), the label-preserving cost \(d_r\) forces \(Q(\mathbb R^n\times\{y\})=1\) and every coupling from \(\PN\) to \(Q\) can only transport mass from the single source point \(z_0\).
Hence \(\nu\le\lambda\) because
\[
\varepsilon^p
\ge
W^p_p(\PN,Q)
=
\int d_r(z_0,z)^p\,dQ(z)
\ge
\nu d_r(z_0,z_1)^p
=
\nu (\varepsilon\lambda^{-1/p})^p
=
\nu\varepsilon^p\lambda^{-1}.
\]

For the loss \(\1\), since \(Q(\mathbb R^n\times\{y\})=1\) and \((\mathbb R^n\times\{y\})\cap(\Z\setminus\mathcal S_{\theta_1})=\{z_1\}\), we have
\begin{equation}
Q(\Z\setminus\mathcal S_{\theta_1})
=
Q\bigl((\mathbb R^n\times\{y\})\cap(\Z\setminus\mathcal S_{\theta_1})\bigr)
=
Q(\{z_1\})
=
\nu
\le
\lambda.
\label{eq:t2-theta1-error-mass-bound-finite-p}
\end{equation}
The law \(P_\lambda=(1-\lambda)\delta_{z_0}+\lambda\delta_{z_1}\) satisfies
\(
W^p_p(\PN,P_\lambda)
=
(1-\lambda)d_r(z_0,z_0)^p+\lambda d_r(z_0,z_1)^p
=
\varepsilon^p,
\)
so it is feasible (see panel~(a) of Figure~\ref{fig:t2-proof-transport}).
Moreover, \(P_\lambda(\Z\setminus\mathcal S_{\theta_1})=\lambda\), and hence for \(\theta_1\),
\[
\mathcal R_p(\theta_1;\varepsilon,\1)
=
\lambda
=
1-\zeta/2
>
1-\zeta.
\]
For \(\theta_2\), since \((\mathbb R^n\times\{y\})\cap(\Z\setminus\mathcal S_{\theta_2})=\emptyset\), we have \(Q(\Z\setminus\mathcal S_{\theta_2})=0\).
Therefore \(\mathcal R_p(\theta_2;\varepsilon,\1)=0\).

For the loss \(\ell\), every feasible law \(Q\) with \(\nu=Q(\{z_1\})\) satisfies
\[
\mathbb E_Q\!\left[\ell(Y,f_{\theta_1}(X))\right]
=
(1-\nu)\ell(y,f_{\underline{\theta}}(x))
+
\nu\ell(y,f_{\theta'}(x)).
\]
Since \(\ell(y,f_{\underline{\theta}}(x))<\ell(y,f_{\theta'}(x))\) by \eqref{eq:t2-low-error-loss-order} and \(\nu\le\lambda\), we have
\[
\mathbb E_Q\!\left[\ell(Y,f_{\theta_1}(X))\right]
\le
(1-\lambda)\ell(y,f_{\underline{\theta}}(x))
+
\lambda\ell(y,f_{\theta'}(x)).
\]
The law \(P_\lambda\) attains equality, so for \(\theta_1\), \(\mathcal R_p(\theta_1;\varepsilon,\ell)=(1-\lambda)\ell(y,f_{\underline{\theta}}(x))+\lambda\ell(y,f_{\theta'}(x))\).
For \(\theta_2\), since \(f_{\theta_2}\) is constant, we have \(\mathcal R_p(\theta_2;\varepsilon,\ell)=\ell(y,f_{\overline{\theta}}(x))\).
By \eqref{eq:t2-mixture-below-high-correct-loss}, we have
\[
\mathcal R_p(\theta_1;\varepsilon,\ell)
<
\ell(y,f_{\overline{\theta}}(x))
=
\mathcal R_p(\theta_2;\varepsilon,\ell).
\]

\emph{Case 2: \(p=\infty\).}
Set \(x_0=0\), \(x_1=\varepsilon e_1\), \(x_2=3\varepsilon e_1\), \(z_0=(x_0,y)\), \(z_1=(x_1,y)\), and \(z_2=(x_2,y)\).
Let \(\PN=\lambda\delta_{z_0}+(1-\lambda)\delta_{z_2}\).
Panel~(b) of Figure~\ref{fig:t2-proof-transport} illustrates the corresponding \(p=\infty\) geometry.
Define a piecewise-constant scoring function \(f_{\theta_1}\) and a constant scoring function \(f_{\theta_2}\) by
\[
f_{\theta_1}(\tilde x)=
\begin{cases}
f_{\theta'}(x), & \tilde x=x_1,\\
f_{\underline{\theta}}(x), & \tilde x\ne x_1,
\end{cases}
\qquad
f_{\theta_2}(\tilde x)\equiv f_{\overline{\theta}}(x).
\]
Let \(Q\) be any law with \(W_\infty(\PN,Q)\le\varepsilon\), and set \(\nu=Q(\{z_1\})\).
Choose a coupling \(\pi\in\Pi(\PN,Q)\) such that \(\operatorname*{ess\,sup}_{(z,z')\sim\pi} d_r(z,z')\le\varepsilon\).
Since \(d_r\) is label-preserving and \(\PN\) is supported on \(\mathbb R^n\times\{y\}\), we have \(Q(\mathbb R^n\times\{y\})=1\).
Moreover, \(d_r(z_2,z_1)=\|x_2-x_1\|_r=2\varepsilon>\varepsilon\), so no mass from \(z_2\) can be coupled to \(z_1\) under \(\pi\).
Therefore all mass assigned by \(Q\) to \(z_1\) must come from the source atom \(z_0\), and hence \(\nu\le\lambda\) because
\[
\nu
=
Q(\{z_1\})
=
\pi(\Z\times\{z_1\})
=
\pi(\{z_0\}\times\{z_1\})
\le
\pi(\{z_0\}\times\Z)
=
\lambda.
\]

For the loss \(\1\), we have \(Q(\Z\setminus\mathcal S_{\theta_1})=Q(\{z_1\})=\nu\le\lambda\) by the same argument used to derive \eqref{eq:t2-theta1-error-mass-bound-finite-p}.
The law \(P_\lambda=\lambda\delta_{z_1}+(1-\lambda)\delta_{z_2}\) is feasible, since it is induced by sending the source mass at \(z_0\) to \(z_1\) and leaving the source mass at \(z_2\) fixed (see panel~(b) of Figure~\ref{fig:t2-proof-transport}).
Moreover, \(P_\lambda(\Z\setminus\mathcal S_{\theta_1})=\lambda\), and hence for \(\theta_1\),
\[
\mathcal R_\infty(\theta_1;\varepsilon,\1)
=
\lambda
=
1-\zeta/2
>
1-\zeta.
\]
For \(\theta_2\), since \((\mathbb R^n\times\{y\})\cap(\Z\setminus\mathcal S_{\theta_2})=\emptyset\), we have \(\mathcal R_\infty(\theta_2;\varepsilon,\1)=0\).

For the loss \(\ell\), every feasible law \(Q\) satisfies
\[
\mathbb E_Q\!\left[\ell(Y,f_{\theta_1}(X))\right]
=
(1-\nu)\ell(y,f_{\underline{\theta}}(x))
+
\nu\ell(y,f_{\theta'}(x)).
\]
Since \(\ell(y,f_{\underline{\theta}}(x))<\ell(y,f_{\theta'}(x))\) by \eqref{eq:t2-low-error-loss-order} and \(\nu\le\lambda\), we have
\[
\mathbb E_Q\!\left[\ell(Y,f_{\theta_1}(X))\right]
\le
(1-\lambda)\ell(y,f_{\underline{\theta}}(x))
+
\lambda\ell(y,f_{\theta'}(x)).
\]
The law \(P_\lambda\) attains equality, so for \(\theta_1\), \(\mathcal R_\infty(\theta_1;\varepsilon,\ell)=(1-\lambda)\ell(y,f_{\underline{\theta}}(x))+\lambda\ell(y,f_{\theta'}(x))\).
For \(\theta_2\), since \(f_{\theta_2}\) is constant, we have \(\mathcal R_\infty(\theta_2;\varepsilon,\ell)=\ell(y,f_{\overline{\theta}}(x))\).
By \eqref{eq:t2-mixture-below-high-correct-loss}, we have
\[
\mathcal R_\infty(\theta_1;\varepsilon,\ell)
<
\ell(y,f_{\overline{\theta}}(x))
=
\mathcal R_\infty(\theta_2;\varepsilon,\ell).
\]
\end{proof}

\section{Additional Experimental Details and Results}
\label{app:experimental-details}

\paragraph{Results by model.}
Tables~\ref{tab:wda-cifar10-model-results} and~\ref{tab:wda-cifar100-imagenet-model-results} report results for all 18 models under the settings in Section~\ref{sec:empirical-evaluation}.
The WDA columns identify the attack used to find support, with classification error optimized over the resulting candidate support sets. 
The Direct column reports classification accuracy from W-PGD \citep{bai2023wasserstein} at \(p=2\) and budget \(\varepsilon\). 
All model checkpoints are taken from the official RobustBench repository \citep{croce2021robustbench}.

\begin{sidewaystable}[p]
\centering
\scriptsize

\providecommand{\normalpaperref}[1]{%
  \mbox{\tiny\citeauthor{#1}~\citeyear{#1}}%
}
\defcitealias{zagoruyko2016wide}{Z\&K}
\providecommand{\zkpaperref}[1]{%
  \mbox{\tiny\citetalias{#1}~\citeyear{#1}}%
}
\providecommand{\withruntime}[2]{}
\renewcommand{\withruntime}[2]{%
  \mbox{%
    \makebox[3.4em][r]{#1}%
    \hspace{0.5em}%
    \makebox[3.4em][l]{{\tiny\mdseries(#2)}}%
  }%
}

\caption{Comparison of WDA and Direct on CIFAR-10.
Entries report classification accuracy (\%) (runtime in minutes).
A dash indicates that no corresponding result is available.}
\label{tab:wda-cifar10-model-results}

\setlength{\tabcolsep}{1.8pt}
\renewcommand{\arraystretch}{1.02}

\begin{tabularx}{\textheight}{@{}lcc*{6}{Y}@{}}
\toprule
\multirow[c]{2}{*}{Paper}
& \multirow[c]{2}{*}{Model}
& \multirow[c]{2}{*}{Clean}
& \multicolumn{5}{c}{WDA}
& \multicolumn{1}{c}{Direct}
\\
\cmidrule(lr){4-8}
\cmidrule(lr){9-9}
& &
& Alg.~\ref{alg:error-point-search}
& APGD-CE
& APGD-DLR
& MM
& W-PGD
& W-PGD
\\

\midrule
\multicolumn{9}{c}{\textbf{\(p=1\)}}
\\
\midrule

\normalpaperref{bartoldson2024adversarial}
& {\tiny WRN-94-16} & 93.68
& \withruntime{\textbf{28.53}}{73m}
& \withruntime{36.77}{207m}
& \withruntime{40.14}{208m}
& \withruntime{33.54}{93m}
& \withruntime{\underline{32.11}}{197m}
& \textendash
\\

\normalpaperref{bartoldson2024adversarial}
& {\tiny WRN-82-8} & 93.11
& \withruntime{\textbf{25.74}}{22m}
& \withruntime{35.00}{58m}
& \withruntime{38.35}{58m}
& \withruntime{31.50}{27m}
& \withruntime{\underline{29.68}}{47m}
& \textendash
\\

\normalpaperref{cui2024decoupled}
& {\tiny WRN-28-10} & 92.16
& \withruntime{\textbf{22.12}}{13m}
& \withruntime{31.90}{32m}
& \withruntime{34.98}{33m}
& \withruntime{28.00}{15m}
& \withruntime{\underline{25.63}}{28m}
& \textendash
\\

\normalpaperref{wang2023better}
& {\tiny WRN-70-16} & 93.25
& \withruntime{\textbf{24.87}}{51m}
& \withruntime{34.37}{151m}
& \withruntime{37.50}{152m}
& \withruntime{30.67}{67m}
& \withruntime{\underline{28.88}}{117m}
& \textendash
\\

\normalpaperref{wang2023better}
& {\tiny WRN-28-10} & 92.44
& \withruntime{\textbf{21.74}}{14m}
& \withruntime{31.35}{32m}
& \withruntime{34.69}{32m}
& \withruntime{27.58}{15m}
& \withruntime{\underline{25.75}}{28m}
& \textendash
\\

\normalpaperref{xu2023exploring}
& {\tiny WRN-28-10} & 93.69
& \withruntime{\textbf{27.69}}{13m}
& \withruntime{33.50}{33m}
& \withruntime{37.11}{33m}
& \withruntime{31.88}{15m}
& \withruntime{\underline{28.76}}{28m}
& \textendash
\\

\normalpaperref{sehwag2022robust}
& {\tiny RN-18} & 84.59
& \withruntime{\textbf{12.89}}{4m}
& \withruntime{19.97}{5m}
& \withruntime{24.80}{5m}
& \withruntime{\underline{17.31}}{3m}
& \withruntime{18.39}{5m}
& \textendash
\\

\normalpaperref{rebuffi2021fixing}
& {\tiny PreActResNet-18} & 83.53
& \withruntime{\textbf{13.62}}{6m}
& \withruntime{20.49}{7m}
& \withruntime{24.24}{8m}
& \withruntime{\underline{17.10}}{4m}
& \withruntime{17.95}{8m}
& \textendash
\\

\normalpaperref{wong2020fast}
& {\tiny PreActResNet-18} & 83.34
& \withruntime{\textbf{2.49}}{5m}
& \withruntime{8.35}{5m}
& \withruntime{15.55}{5m}
& \withruntime{\underline{6.69}}{3m}
& \withruntime{9.55}{5m}
& \textendash
\\

\zkpaperref{zagoruyko2016wide}
& {\tiny WRN-28-10} & 94.78
& \withruntime{\textbf{0.00}}{6m}
& \withruntime{\textbf{0.00}}{32m}
& \withruntime{\textbf{0.00}}{32m}
& \withruntime{\textbf{0.00}}{10m}
& \withruntime{\textbf{0.00}}{28m}
& \textendash
\\

\midrule
\multicolumn{9}{c}{\textbf{\(p=2\)}}
\\
\midrule

\normalpaperref{bartoldson2024adversarial}
& {\tiny WRN-94-16} & 93.68
& \withruntime{\textbf{42.57}}{73m}
& \withruntime{54.74}{207m}
& \withruntime{57.43}{208m}
& \withruntime{50.70}{93m}
& \withruntime{\underline{46.35}}{197m}
& 58.72
\\

\normalpaperref{bartoldson2024adversarial}
& {\tiny WRN-82-8} & 93.11
& \withruntime{\textbf{39.47}}{22m}
& \withruntime{52.87}{58m}
& \withruntime{55.39}{58m}
& \withruntime{48.49}{27m}
& \withruntime{\underline{43.83}}{47m}
& 56.80
\\

\normalpaperref{cui2024decoupled}
& {\tiny WRN-28-10} & 92.16
& \withruntime{\textbf{35.17}}{13m}
& \withruntime{49.30}{32m}
& \withruntime{51.53}{33m}
& \withruntime{44.54}{15m}
& \withruntime{\underline{39.06}}{28m}
& 50.96
\\

\normalpaperref{wang2023better}
& {\tiny WRN-70-16} & 93.25
& \withruntime{\textbf{38.32}}{51m}
& \withruntime{52.12}{151m}
& \withruntime{54.46}{152m}
& \withruntime{47.49}{67m}
& \withruntime{\underline{42.80}}{117m}
& 55.57
\\

\normalpaperref{wang2023better}
& {\tiny WRN-28-10} & 92.44
& \withruntime{\textbf{34.73}}{14m}
& \withruntime{48.63}{32m}
& \withruntime{51.22}{32m}
& \withruntime{44.00}{15m}
& \withruntime{\underline{39.26}}{28m}
& 52.34
\\

\normalpaperref{xu2023exploring}
& {\tiny WRN-28-10} & 93.69
& \withruntime{\textbf{40.76}}{13m}
& \withruntime{50.53}{33m}
& \withruntime{53.07}{33m}
& \withruntime{48.23}{15m}
& \withruntime{\underline{41.55}}{28m}
& 52.38
\\

\normalpaperref{sehwag2022robust}
& {\tiny RN-18} & 84.59
& \withruntime{\textbf{25.19}}{4m}
& \withruntime{36.42}{5m}
& \withruntime{40.59}{5m}
& \withruntime{32.74}{3m}
& \withruntime{\underline{31.42}}{5m}
& 46.71
\\

\normalpaperref{rebuffi2021fixing}
& {\tiny PreActResNet-18} & 83.53
& \withruntime{\textbf{25.96}}{6m}
& \withruntime{36.67}{7m}
& \withruntime{40.01}{8m}
& \withruntime{32.43}{4m}
& \withruntime{\underline{30.56}}{8m}
& 44.26
\\

\normalpaperref{wong2020fast}
& {\tiny PreActResNet-18} & 83.34
& \withruntime{\textbf{13.00}}{5m}
& \withruntime{22.77}{5m}
& \withruntime{30.29}{5m}
& \withruntime{\underline{20.43}}{3m}
& \withruntime{21.48}{5m}
& 37.87
\\

\zkpaperref{zagoruyko2016wide}
& {\tiny WRN-28-10} & 94.78
& \withruntime{\textbf{0.00}}{6m}
& \withruntime{\textbf{0.00}}{32m}
& \withruntime{\textbf{0.00}}{32m}
& \withruntime{\textbf{0.00}}{10m}
& \withruntime{\textbf{0.00}}{28m}
& \textbf{0.00}
\\

\midrule
\multicolumn{9}{c}{\textbf{\(p=\infty\)}}
\\
\midrule

\normalpaperref{bartoldson2024adversarial}
& {\tiny WRN-94-16} & 93.68
& \withruntime{76.05}{73m}
& \withruntime{76.14}{207m}
& \withruntime{\underline{74.31}}{208m}
& \withruntime{\textbf{73.82}}{93m}
& \withruntime{77.82}{197m}
& \textendash
\\

\normalpaperref{bartoldson2024adversarial}
& {\tiny WRN-82-8} & 93.11
& \withruntime{74.04}{22m}
& \withruntime{74.24}{58m}
& \withruntime{\underline{72.48}}{58m}
& \withruntime{\textbf{71.72}}{27m}
& \withruntime{76.03}{47m}
& \textendash
\\

\normalpaperref{cui2024decoupled}
& {\tiny WRN-28-10} & 92.16
& \withruntime{70.09}{13m}
& \withruntime{70.54}{32m}
& \withruntime{\underline{68.66}}{33m}
& \withruntime{\textbf{67.84}}{15m}
& \withruntime{72.24}{28m}
& \textendash
\\

\normalpaperref{wang2023better}
& {\tiny WRN-70-16} & 93.25
& \withruntime{73.29}{51m}
& \withruntime{73.51}{151m}
& \withruntime{\underline{71.57}}{152m}
& \withruntime{\textbf{70.84}}{67m}
& \withruntime{75.58}{117m}
& \textendash
\\

\normalpaperref{wang2023better}
& {\tiny WRN-28-10} & 92.44
& \withruntime{69.73}{14m}
& \withruntime{70.21}{32m}
& \withruntime{\underline{68.25}}{32m}
& \withruntime{\textbf{67.46}}{15m}
& \withruntime{72.30}{28m}
& \textendash
\\

\normalpaperref{xu2023exploring}
& {\tiny WRN-28-10} & 93.69
& \withruntime{72.13}{13m}
& \withruntime{\underline{67.08}}{33m}
& \withruntime{68.89}{33m}
& \withruntime{\textbf{65.84}}{15m}
& \withruntime{73.03}{28m}
& \textendash
\\

\normalpaperref{sehwag2022robust}
& {\tiny RN-18} & 84.59
& \withruntime{58.54}{4m}
& \withruntime{58.39}{5m}
& \withruntime{\underline{57.58}}{5m}
& \withruntime{\textbf{56.01}}{3m}
& \withruntime{62.28}{5m}
& \textendash
\\

\normalpaperref{rebuffi2021fixing}
& {\tiny PreActResNet-18} & 83.53
& \withruntime{58.32}{6m}
& \withruntime{59.59}{7m}
& \withruntime{\underline{57.29}}{8m}
& \withruntime{\textbf{56.84}}{4m}
& \withruntime{60.78}{8m}
& \textendash
\\

\normalpaperref{wong2020fast}
& {\tiny PreActResNet-18} & 83.34
& \withruntime{47.91}{5m}
& \withruntime{\underline{45.90}}{5m}
& \withruntime{46.98}{5m}
& \withruntime{\textbf{43.93}}{3m}
& \withruntime{53.88}{5m}
& \textendash
\\

\zkpaperref{zagoruyko2016wide}
& {\tiny WRN-28-10} & 94.78
& \withruntime{\underline{0.10}}{6m}
& \withruntime{\textbf{0.00}}{32m}
& \withruntime{\textbf{0.00}}{32m}
& \withruntime{\textbf{0.00}}{10m}
& \withruntime{0.19}{28m}
& \textendash
\\

\bottomrule
\end{tabularx}
\end{sidewaystable}

\begin{sidewaystable}[p]
\centering
\scriptsize

\providecommand{\paperref}[1]{%
  \mbox{\tiny\citeauthor{#1}~\citeyear{#1}}%
}
\providecommand{\withruntime}[2]{}
\renewcommand{\withruntime}[2]{%
  \mbox{%
    \makebox[3.4em][r]{#1}%
    \hspace{0.5em}%
    \makebox[3.4em][l]{{\tiny\mdseries(#2)}}%
  }%
}

\caption{Comparison of WDA and Direct on CIFAR-100 and ImageNet.
Entries report classification accuracy (\%) (runtime in minutes).
A dash indicates that no corresponding result is available.}
\label{tab:wda-cifar100-imagenet-model-results}

\setlength{\tabcolsep}{1.8pt}
\renewcommand{\arraystretch}{1.02}

\begin{tabularx}{\textheight}{@{}lcc*{6}{Y}@{}}
\toprule
\multirow[c]{2}{*}{Paper}
& \multirow[c]{2}{*}{Model}
& \multirow[c]{2}{*}{Clean}
& \multicolumn{5}{c}{WDA}
& \multicolumn{1}{c}{Direct}
\\
\cmidrule(lr){4-8}
\cmidrule(lr){9-9}
& &
& Alg.~\ref{alg:error-point-search}
& APGD-CE
& APGD-DLR
& MM
& W-PGD
& W-PGD
\\

\midrule
\multicolumn{9}{c}{CIFAR-100, \(r=\infty\), \(\varepsilon=8/255\)}
\\
\midrule
\multicolumn{9}{c}{\(p=1\)}
\\
\midrule

\paperref{cui2024decoupled}
& {\tiny WRN-28-10} & 73.85
& \withruntime{\underline{2.29}}{14m}
& \withruntime{6.03}{26m}
& \withruntime{10.05}{26m}
& \withruntime{\textbf{1.86}}{12m}
& \withruntime{3.89}{27m}
& \textendash
\\

\paperref{wang2023better}
& {\tiny WRN-70-16} & 75.22
& \withruntime{\textbf{3.98}}{55m}
& \withruntime{10.62}{124m}
& \withruntime{13.21}{124m}
& \withruntime{\underline{5.19}}{51m}
& \withruntime{7.46}{117m}
& \textendash
\\

\paperref{wang2023better}
& {\tiny WRN-28-10} & 72.58
& \withruntime{\underline{2.90}}{14m}
& \withruntime{6.73}{25m}
& \withruntime{9.50}{26m}
& \withruntime{\textbf{1.69}}{11m}
& \withruntime{4.13}{27m}
& \textendash
\\

\paperref{addepalli2022efficient}
& {\tiny ResNet-18} & 65.45
& \withruntime{2.03}{4m}
& \withruntime{1.25}{5m}
& \withruntime{1.11}{5m}
& \withruntime{\underline{0.99}}{2m}
& \withruntime{\textbf{0.00}}{4m}
& \textendash
\\

\midrule
\multicolumn{9}{c}{\(p=2\)}
\\
\midrule

\paperref{cui2024decoupled}
& {\tiny WRN-28-10} & 73.85
& \withruntime{\textbf{9.24}}{14m}
& \withruntime{21.29}{26m}
& \withruntime{25.12}{26m}
& \withruntime{\underline{14.70}}{12m}
& \withruntime{14.81}{27m}
& 24.59
\\

\paperref{wang2023better}
& {\tiny WRN-70-16} & 75.22
& \withruntime{\textbf{13.32}}{55m}
& \withruntime{26.35}{124m}
& \withruntime{28.47}{124m}
& \withruntime{\underline{18.92}}{51m}
& \withruntime{19.00}{117m}
& 28.96
\\

\paperref{wang2023better}
& {\tiny WRN-28-10} & 72.58
& \withruntime{\textbf{9.68}}{14m}
& \withruntime{21.76}{25m}
& \withruntime{24.31}{26m}
& \withruntime{\underline{14.10}}{11m}
& \withruntime{14.87}{27m}
& 24.90
\\

\paperref{addepalli2022efficient}
& {\tiny ResNet-18} & 65.45
& \withruntime{\textbf{4.09}}{4m}
& \withruntime{13.12}{5m}
& \withruntime{14.52}{5m}
& \withruntime{\underline{6.38}}{2m}
& \withruntime{8.14}{4m}
& 14.52
\\

\midrule
\multicolumn{9}{c}{\(p=\infty\)}
\\
\midrule

\paperref{cui2024decoupled}
& {\tiny WRN-28-10} & 73.85
& \withruntime{40.58}{14m}
& \withruntime{43.84}{26m}
& \withruntime{\underline{40.32}}{26m}
& \withruntime{\textbf{39.43}}{12m}
& \withruntime{44.79}{27m}
& \textendash
\\

\paperref{wang2023better}
& {\tiny WRN-70-16} & 75.22
& \withruntime{44.44}{55m}
& \withruntime{48.18}{124m}
& \withruntime{\underline{43.88}}{124m}
& \withruntime{\textbf{42.86}}{51m}
& \withruntime{48.09}{117m}
& \textendash
\\

\paperref{wang2023better}
& {\tiny WRN-28-10} & 72.58
& \withruntime{39.82}{14m}
& \withruntime{44.04}{25m}
& \withruntime{\underline{39.61}}{26m}
& \withruntime{\textbf{38.91}}{11m}
& \withruntime{43.75}{27m}
& \textendash
\\

\paperref{addepalli2022efficient}
& {\tiny ResNet-18} & 65.45
& \withruntime{29.94}{4m}
& \withruntime{33.55}{5m}
& \withruntime{\underline{28.85}}{5m}
& \withruntime{\textbf{27.90}}{2m}
& \withruntime{33.68}{4m}
& \textendash
\\

\midrule
\multicolumn{9}{c}{ImageNet, \(r=\infty\), \(\varepsilon=4/255\)}
\\
\midrule
\multicolumn{9}{c}{\(p=1\)}
\\
\midrule

\paperref{xu2026mimir}
& {\tiny Swin-L} & 78.62
& \withruntime{\textbf{16.41}}{110m}
& \withruntime{25.05}{76m}
& \withruntime{27.18}{77m}
& \withruntime{\underline{21.64}}{40m}
& \withruntime{23.77}{122m}
& \textendash
\\

\paperref{liu2025comprehensive}
& {\tiny ConvNeXt-B} & 76.02
& \withruntime{\textbf{14.11}}{42m}
& \withruntime{22.51}{31m}
& \withruntime{24.43}{32m}
& \withruntime{\underline{19.37}}{18m}
& \withruntime{22.24}{65m}
& \textendash
\\

\paperref{salman2020do}
& {\tiny WRN-50-2} & 68.46
& \withruntime{\textbf{3.54}}{29m}
& \withruntime{5.68}{27m}
& \withruntime{10.37}{27m}
& \withruntime{\underline{4.05}}{14m}
& \withruntime{8.10}{59m}
& \textendash
\\

\paperref{salman2020do}
& {\tiny ResNet-18} & 52.92
& \withruntime{1.30}{10m}
& \withruntime{\underline{0.62}}{7m}
& \withruntime{0.96}{7m}
& \withruntime{\underline{0.62}}{5m}
& \withruntime{\textbf{0.00}}{44m}
& \textendash
\\

\midrule
\multicolumn{9}{c}{\(p=2\)}
\\
\midrule

\paperref{xu2026mimir}
& {\tiny Swin-L} & 78.62
& \withruntime{\textbf{29.71}}{110m}
& \withruntime{42.57}{76m}
& \withruntime{44.17}{77m}
& \withruntime{38.16}{40m}
& \withruntime{\underline{36.98}}{122m}
& 48.68
\\

\paperref{liu2025comprehensive}
& {\tiny ConvNeXt-B} & 76.02
& \withruntime{\textbf{26.90}}{42m}
& \withruntime{39.77}{31m}
& \withruntime{40.94}{32m}
& \withruntime{35.47}{18m}
& \withruntime{\underline{35.01}}{65m}
& 45.84
\\

\paperref{salman2020do}
& {\tiny WRN-50-2} & 68.46
& \withruntime{\textbf{11.56}}{29m}
& \withruntime{21.57}{27m}
& \withruntime{25.91}{27m}
& \withruntime{\underline{18.61}}{14m}
& \withruntime{19.94}{59m}
& 30.62
\\

\paperref{salman2020do}
& {\tiny ResNet-18} & 52.92
& \withruntime{\textbf{2.65}}{10m}
& \withruntime{9.12}{7m}
& \withruntime{11.81}{7m}
& \withruntime{\underline{5.73}}{5m}
& \withruntime{8.29}{44m}
& 15.26
\\

\midrule
\multicolumn{9}{c}{\(p=\infty\)}
\\
\midrule

\paperref{xu2026mimir}
& {\tiny Swin-L} & 78.62
& \withruntime{\underline{60.94}}{110m}
& \withruntime{61.64}{76m}
& \withruntime{61.62}{77m}
& \withruntime{\textbf{60.02}}{40m}
& \withruntime{64.00}{122m}
& \textendash
\\

\paperref{liu2025comprehensive}
& {\tiny ConvNeXt-B} & 76.02
& \withruntime{\underline{57.00}}{42m}
& \withruntime{58.04}{31m}
& \withruntime{58.52}{32m}
& \withruntime{\textbf{56.18}}{18m}
& \withruntime{61.12}{65m}
& \textendash
\\

\paperref{salman2020do}
& {\tiny WRN-50-2} & 68.46
& \withruntime{40.90}{29m}
& \withruntime{\underline{40.86}}{27m}
& \withruntime{41.70}{27m}
& \withruntime{\textbf{38.98}}{14m}
& \withruntime{46.12}{59m}
& \textendash
\\

\paperref{salman2020do}
& {\tiny ResNet-18} & 52.92
& \withruntime{\underline{26.80}}{10m}
& \withruntime{29.34}{7m}
& \withruntime{27.34}{7m}
& \withruntime{\textbf{25.66}}{5m}
& \withruntime{30.54}{44m}
& \textendash
\\

\bottomrule
\end{tabularx}
\end{sidewaystable}

\paragraph{Effect of bisection refinement.}
We compare APGD-CE, APGD-DLR, MM, and W-PGD within WDA with and without 12 bisection steps according to \eqref{eq:candidate-support-refinement}.
We optimize classification error under the same budget \(\varepsilon\) in both evaluations.
Figure~\ref{fig:wda-bisection-ablation} reports mean classification accuracy, averaged equally over the models within each dataset.
For \(p=\infty\), bisection reduces mean accuracy by \(1\)\textendash\(2\) percentage points for W-PGD, compared with at most
0.03 percentage points for the other attacks.

\begin{figure}[htb]
\centering
\includegraphics[width=0.99\linewidth]{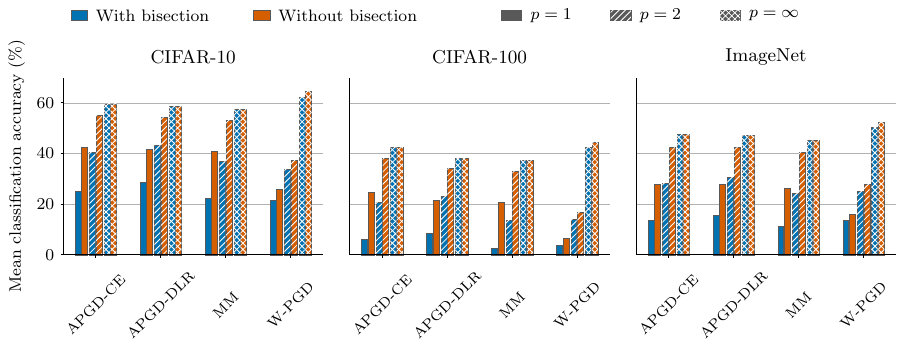}
\caption{Effect of bisection refinement on WDA. Lower accuracy indicates stronger attack.}
\label{fig:wda-bisection-ablation}
\end{figure}

\FloatBarrier

\bibliography{refs}

\end{document}